\pdfoutput=1

\documentclass[11pt]{article}

\usepackage{ACL2023}
\usepackage{custom}

\newcommand{\mymacro}[1]{{ #1}}

\newcommand{\defn}[1]{\textbf{#1}}

\newcommand{\paroutline}[3][false]{%
    \ifnum\pdfstrcmp{#1}{true}=0
        #3
    \else
        [\textit{\textcolor{DiverseMagenta}{#2}}] \textcolor{AccentBlue}{#3}
    \fi
}

\newcommand{\pdens}{{\mymacro{ p}}}

\newcommand{\qdens}{{\mymacro{ q}}}

\newcommand{\ind}[1]{\mathbbm{1} \left\{ #1 \right\}}

\newcommand{\digit}[2]{\mymacro{\evd_{#2}\left( #1 \right)}}
\newcommand{\digitVec}[1]{\mymacro{\vd\left( #1 \right)}}

\newcommand{\Q}{{\mymacro{ \mathbb{Q}}}}
\newcommand{\R}{{\mymacro{ \mathbb{R}}}}
\newcommand{\Rex}{{\mymacro{ \overline{\R}}}}

\newcommand{\B}{{\mymacro{ \mathbb{B}}}}

\newcommand{\func}{{\mymacro{ f}}}
\newcommand{\funcg}{{\mymacro{ g}}}
\newcommand{\vfunc}{{\mymacro{ \boldsymbol{f}}}}

\newcommand{\norm}[1]{{\mymacro{ \left\lVert #1 \right\rVert}}}
\newcommand{\abs}[1]{{\mymacro{ \left| #1 \right|}}}

\newcommand{\inner}[2]{{\mymacro{ #1^\top #2}}}
\newcommand{\innerProd}[2]{{\mymacro{ \left\langle #1, #2 \right\rangle }}}

\newcommand{\alphabet}{{\mymacro{ \Sigma}}}

\newcommand{\eosalphabet}{{\mymacro{ \overline{\alphabet}}}}
\newcommand{\bosalphabet}{{\mymacro{ \underline{\alphabet}}}}
\newcommand{\lang}{{\mymacro{\sL}}}
\newcommand{\kleene}[1]{{\mymacro{#1^*}}}

\newcommand{\str}{{\mymacro{\boldsymbol{y}}}}

\newcommand{\strlt}{{\mymacro{ \str_{<\tstep}}}}

\newcommand{\strlen}{{\mymacro{T}}}

\newcommand{\sym}{{\mymacro{y}}}

\newcommand{\syma}{{\mymacro{a}}}
\newcommand{\symb}{{\mymacro{b}}}

\newcommand{\defeq}{\mathrel{\stackrel{\textnormal{\tiny def}}{=}}}

\newcommand{\NTo}[1]{{\mymacro{\left[ #1 \right]}}}

\newcommand{\set}[1]{{\mymacro{\left\{ #1 \right\}}}}

\newcommand{\justification}[1]{%
    \refstepcounter{equation}%
    \tag{\theequation \textcolor{black!50}{, \footnotesize{#1}}}
}

\newcommand{\idxn}{{\mymacro{ n}}}
\newcommand{\idxd}{{\mymacro{ d}}}
\newcommand{\idxi}{{\mymacro{ i}}}
\newcommand{\idxj}{{\mymacro{ j}}}
\newcommand{\idxk}{{\mymacro{ k}}}

\newcommand{\idxm}{{\mymacro{ m}}}

\newcommand{\nsymbols}{{\mymacro{ |\alphabet|}}}

\newcommand{\bosnsymbols}{{\mymacro{ |\bosalphabet|}}}

\newcommand{\tstep}{{\mymacro{ t}}}
\newcommand{\tstepminus}{{\mymacro{\tstep - 1}}}
\newcommand{\finaltstep}{{\mymacro{ T}}}

\newcommand{\pLM}{\mymacro{\pdens}}
\newcommand{\qLM}{\mymacro{\qdens}}
\newcommand{\pLNSM}{\mymacro{\pdens}}

\newcommand{\bos}{{\mymacro{\textsc{bos}}}}
\newcommand{\eos}{{\mymacro{\textsc{eos}}}}
\newcommand{\ngr}{{\mymacro{ \textit{n}}}}
\newcommand{\ngram}{{\mymacro{ \textit{n}-gram}}\xspace}

\newcommand{\onehot}[1]{{\mymacro{ \llbracket#1\rrbracket}}}

\newcommand{\inEmbedding}{{\mymacro{ \vr}}}

\newcommand{\embedMtx}{{\mymacro{ \mE}}}
\newcommand{\eembedMtx}{{\mymacro{ \emE}}}

\newcommand{\symt}{{\mymacro{ \sym_{\tstep}}}}

\newcommand{\symtminus}{{\mymacro{ \sym_{\tstep-1}}}}

\newcommand{\symone}{{\mymacro{ \sym_{1}}}}

\newcommand{\zero}{{\mymacro{\mathbf{0}}}}

\newcommand{\outMtx}{{\mymacro{ \mE}}}
\newcommand{\eOutMtx}{{\mymacro{ \eembedMtx}}}

\newcommand{\hiddDim}{{\mymacro{ D}}}

\newcommand{\staticRepr}{{\mymacro{\mathcal{R}}}}

\newcommand{\symordering}{{\mymacro{ m}}}

\newcommand{\anOrdering}{{\mymacro{ s}}}

\newcommand{\enc}{{\mymacro{\mathsf{enc}}}}

\newcommand{\mlp}{{\mymacro{\mathrm{MLP}}}}
\newcommand{\simplexFun}[1]{{\mymacro{ \boldsymbol{\Delta}}^{#1}}}

\newcommand{\Simplexdminus}{{\mymacro{ \boldsymbol{\Delta}^{D-1}}}}

\newcommand{\negterm}[1]{{\mymacro{ {\raise.17ex\hbox{$\scriptstyle\sim$}} #1}}}

\newcommand{\ifcondition}{\textbf{if }}
\newcommand{\otherwisecondition}{\textbf{otherwise }}

\newcommand{\ignore}[1]{}
\newcommand{\expandLater}[1]{}

\newcommand{\tfheadnum}{\mymacro{H}}

\newcommand{\qTransf}{\mymacro{Q}}
\newcommand{\kTransf}{\mymacro{K}}
\newcommand{\vTransf}{\mymacro{V}}
\newcommand{\oTransf}{\mymacro{O}}
\newcommand{\fTransf}{\mymacro{F}}

\newcommand{\oTransfFun}[1]{\mymacro{\oTransf\left(#1\right)}}

\newcommand{\attn}{\mymacro{\texttt{Att}}}

\newcommand{\tfheadCombine}{\mymacro{\mathcal{H}}}

\newcommand{\transformernetwork}{\mymacro{\mathcal{T}}}
\newcommand{\tf}{\mymacro{\mathcal{T}}}
\newcommand{\tfFun}[1]{\tf\left(#1\right)}
\newcommand{\tfpLM}{\mymacro{\pLM_\transformernetwork}}
\newcommand{\tfencfun}{\mymacro{\enc}}

\newcommand{\tfscorefun}{\mymacro{\func}}

\newcommand{\hardmax}{\mymacro{\mathrm{hardmax}}}
\newcommand{\hardmaxAvg}{\mymacro{\hardmax}}

\newcommand{\tflayer}{\mymacro{\mathrm{\mathcal{L}}}}
\newcommand{\tflayerinputmat}{\mymacro{\mX}}
\newcommand{\tflayerinputsy}{\mymacro{\vx}}

\newcommand{\tflayeroutputsy}{\mymacro{\vz}}

\newcommand{\tflayeridx}{\mymacro{\ell}}
\newcommand{\tfnumlayer}{\mymacro{L}}

\newcommand{\posInEmbedding}{\mymacro{\inEmbedding}}
\newcommand{\posInEmbeddingFun}[1]{\mymacro{\posInEmbedding\left(#1\right)}}

\def\1{\mathbf{1}}

\def\rvp{{{\mymacro{ \mathbf{p}}}}}

\def\va{{{\mymacro{ \mathbf{a}}}}}
\def\vb{{{\mymacro{ \mathbf{b}}}}}

\def\vd{{{\mymacro{ \mathbf{d}}}}}

\def\vk{{{\mymacro{ \mathbf{k}}}}}

\def\vq{{{\mymacro{ \mathbf{q}}}}}
\def\vr{{{\mymacro{ \mathbf{r}}}}}
\def\vs{{{\mymacro{ \mathbf{s}}}}}

\def\vu{{{\mymacro{ \mathbf{u}}}}}
\def\vv{{{\mymacro{ \mathbf{v}}}}}

\def\vx{{{\mymacro{ \mathbf{x}}}}}

\def\vz{{{\mymacro{ \mathbf{z}}}}}

\def\eva{{{\mymacro{ a}}}}
\def\evb{{{\mymacro{ b}}}}

\def\evd{{{\mymacro{ d}}}}

\def\evs{{{\mymacro{ s}}}}

\def\evu{{{\mymacro{ u}}}}
\def\evv{{{\mymacro{ v}}}}

\def\evx{{{\mymacro{ x}}}}

\def\evz{{{\mymacro{ z}}}}

\def\mE{{{\mymacro{ \mathbf{E}}}}}

\def\mI{{{\mymacro{ \mathbf{I}}}}}

\def\mK{{{\mymacro{ \mathbf{K}}}}}

\def\mO{{{\mymacro{ \mathbf{O}}}}}

\def\mQ{{{\mymacro{ \mathbf{Q}}}}}

\def\mV{{{\mymacro{ \mathbf{V}}}}}
\def\mW{{{\mymacro{ \mathbf{W}}}}}
\def\mX{{{\mymacro{ \mathbf{X}}}}}

\def\sI{{{\mymacro{ \mathcal{I}}}}}

\def\sL{{{\mymacro{ \mathcal{L}}}}}

\def\emE{{\mymacro{ E}}}

\newcommand{\N}{{\mymacro{ \mathbb{N}}}}

\newcommand{\projfunc}{{\mymacro{\boldsymbol{\pi}}}}

\newcommand{\softmax}{{\mymacro{ \mathrm{softmax}}}}
\newcommand{\sparsemax}{{\mymacro{ \mathrm{sparsemax}}}}
\newcommand{\ReLU}{{\mymacro{ \mathrm{ReLU}}}}
\newcommand{\softmaxfunc}[2]{{\mymacro{ \mathrm{softmax}\!\left(#1\right)_{#2}}}} 
\newcommand{\sparsemaxfunc}[2]{{\mymacro{ \mathrm{sparsemax}\!\left(#1\right)_{#2}}}} 
\newcommand{\ReLUfunc}[1]{{\mymacro{ \ReLU\!\left(#1\right)}}} 

\DeclareMathOperator*{\argmax}{{\mymacro{ argmax}}}
\DeclareMathOperator*{\argmin}{{\mymacro{ argmin}}}

\newcommand{\bigO}[1]{{\mymacro{ \mathcal{O}\left(#1\right)}}}

\title{Transformers Can Represent $n$-gram Language Models}

\author{
Anej Svete%
~\;~\;~Ryan Cotterell\\
\texttt{\{\href{asvete@inf.ethz.ch}{asvete}, \href{ryan.cotterell@inf.ethz.ch}{ryan.cotterell}\}@inf.ethz.ch}\\
    {%
\setlength{\fboxsep}{2.5pt}%
\setlength{\fboxrule}{2.5pt}%
\fcolorbox{white}{white}{
    \includegraphics[width=.15\linewidth]{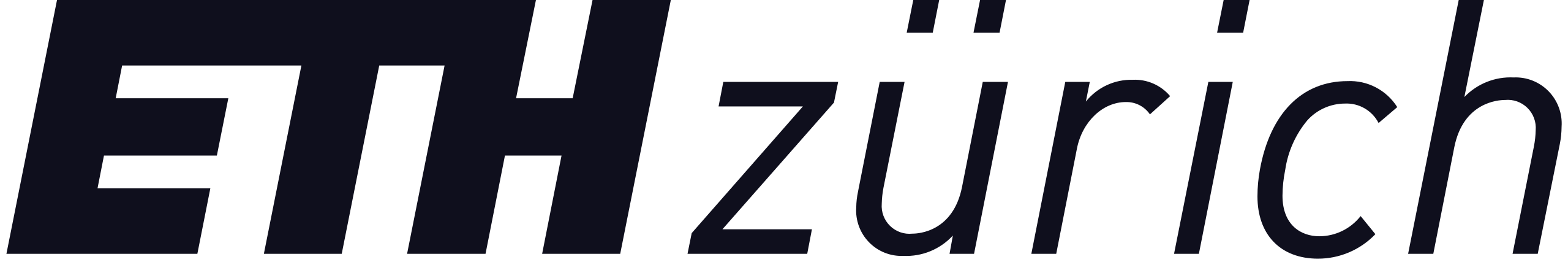}
}
}}

\begin{document}
\maketitle

\begin{abstract}
    Existing work has analyzed the representational capacity of the transformer architecture by means of formal models of computation.
    However, the focus so far has been on analyzing the architecture in terms of language \emph{acceptance}.
    We contend that this is an ill-suited problem in the study of \emph{language models} (LMs), which are definitionally \emph{probability distributions} over strings.
    In this paper, we focus on the relationship between transformer LMs and \ngram LMs, a simple and historically relevant class of language models. 
    We show that transformer LMs using the hard or sparse attention mechanisms can exactly represent any \ngram LM, giving us a concrete lower bound on their probabilistic representational capacity.
    This provides a first step towards understanding the mechanisms that transformer LMs can use to represent probability distributions over strings.
    
    \vspace{0.5em}
    \hspace{.5em}\includegraphics[width=1.25em,height=1.25em]{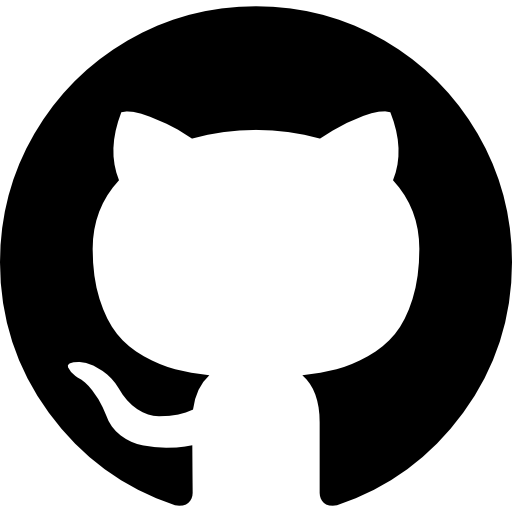}\hspace{.75em}\parbox{\dimexpr\linewidth-2\fboxsep-2\fboxrule}{\url{https://github.com/rycolab/transformer-ngrams}}
    \vspace{-.5em}
\end{abstract}

\section{Introduction} \label{sec:intro}
Neural language models (LMs) have become the backbone of many NLP systems. 
Their widespread adoption has prompted a plethora of theoretical work investigating what they can and cannot do by studying their representational capacity.
Most state-of-the-art LMs are based on the transformer architecture \citep{Vaswani2017}, whose theoretical abilities and limitations have been studied extensively; see, e.g., the survey by \citet{strobl2023transformers}. 
But, many questions remain unanswered.
Most existing work studies the architecture in terms of binary language recognition.
This introduces a category error between the object of study---an LM, which is definitionally a \emph{distribution} over strings---and the theoretical abstraction---a \emph{set} of strings.
To amend this discrepancy, we ask: What classes of probability distributions over strings can transformer LMs represent?\looseness=-1

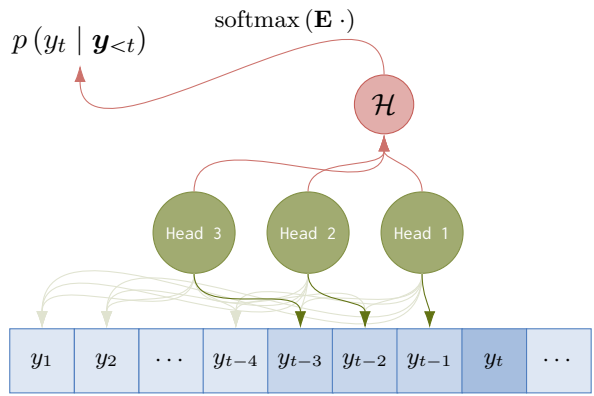
\begin{figure}
    \centering

    \begin{tikzpicture}[
        tape node/.style={draw=ETHBlue!80,minimum size=0.85cm,fill=ETHBlue!20},
        head node/.style={draw=ETHGreen!80,circle,minimum size=0.65cm,fill=ETHGreen!60,text=white},
        attn arrow/.style={-{Latex[length=2.25mm,width=1.5mm]},ETHGreen!100},
        comb arrow/.style={-{Latex[length=2.25mm,width=1.5mm]},ETHRed!70},
        comb node/.style={draw=ETHRed!80,circle,minimum size=0.75cm,fill=ETHRed!40},
        ]

        \foreach \i/\y in {0/$\sym_1$,1/$\sym_2$,2/$\cdots$,3/$\sym_{\tstep-4}$,4/$\sym_{\tstep-3}$,5/$\sym_{\tstep-2}$,6/$\sym_{\tstep-1}$,7/$\symt$,8/$\cdots$} {
                \ifnum \i=7
                    \node[tape node,fill=ETHBlue!40] (tape-\i) at (0.85*\i,0) {\footnotesize \y};
                \else
                    \ifnum \i>3
                        \ifnum \i<7
                            \node[tape node,fill=ETHBlue!25] (tape-\i) at (0.85*\i,0) {\footnotesize \y};
                        \fi
                    \else
                        \node[tape node,fill=ETHBlue!15] (tape-\i) at (0.85*\i,0) {\footnotesize \y};
                    \fi
                    \ifnum \i>7
                        \node[tape node,fill=ETHBlue!15] (tape-\i) at (0.85*\i,0) {\footnotesize \y};
                    \fi
                \fi
            }

        \node[head node] (head-1) at (2,1.7) {\scriptsize \texttt{Head 3}};
        \node[head node] (head-2) at (3.5,1.7) {\scriptsize \texttt{Head 2}};
        \node[head node] (head-3) at (5,1.7) {\scriptsize \texttt{Head 1}};

        \draw[attn arrow, ETHGreen!20] (head-1) to[out=270,in=90] (tape-0.north);
        \draw[attn arrow, ETHGreen!20] (head-1) to[out=270,in=90] (tape-1.north);
        \draw[attn arrow, ETHGreen!20] (head-1) to[out=270,in=90] (tape-3.north);
        \draw[attn arrow, ETHGreen!20] (head-1) to[out=270,in=90] (tape-5.north);
        \draw[attn arrow, ETHGreen!20] (head-2) to[out=270,in=90] (tape-0.north);
        \draw[attn arrow, ETHGreen!20] (head-2) to[out=270,in=90] (tape-1.north);
        \draw[attn arrow, ETHGreen!20] (head-2) to[out=270,in=90] (tape-3.north);
        \draw[attn arrow, ETHGreen!20] (head-2) to[out=270,in=90] (tape-4.north);
        \draw[attn arrow, ETHGreen!20] (head-3) to[out=270,in=90] (tape-0.north);
        \draw[attn arrow, ETHGreen!20] (head-3) to[out=270,in=90] (tape-1.north);
        \draw[attn arrow, ETHGreen!20] (head-3) to[out=270,in=90] (tape-3.north);
        \draw[attn arrow, ETHGreen!20] (head-3) to[out=270,in=90] (tape-4.north);
        \draw[attn arrow, ETHGreen!20] (head-3) to[out=270,in=90] (tape-5.north);
        \draw[attn arrow] (head-1) to[out=270,in=90] (tape-4.north);
        \draw[attn arrow] (head-2) to[out=270,in=90] (tape-5.north);
        \draw[attn arrow] (head-3) to[out=270,in=90] (tape-6.north);

        \node[comb node] (combiner) at (4.5,3.4) {$\tfheadCombine$};

        \draw[comb arrow] (head-1.north) to[out=90,in=270] (combiner.south);
        \draw[comb arrow] (head-2.north) to[out=90,in=270] (combiner.south);
        \draw[comb arrow] (head-3.north) to[out=90,in=270] (combiner.south);

        \node[fill=none] (out) at (0.5,4.25) {$\pLNSM\left(\sym_\tstep\mid\str_{<\tstep}\right)$};

        \draw[comb arrow] (combiner.north) to[out=90,in=270] (out.south) ;
        \node[fill=none] (out) at (3.2,4.5) {\footnotesize $\softmax\left(\outMtx \; \cdot\right)$};

    \end{tikzpicture}
    \caption{A transformer LM can simulate a $4$-gram LM using \textcolor{ETHGreen}{3 heads}. The stronger arrows from the heads to the symbols show where the heads focus their attention.}
    \label{fig:transformer-n-gram-label}
    \vspace{-10pt}
\end{figure}

Formal models of probabilistic computation provide a natural, well-understood, and precise framework for studying the classes of probability distributions language models can represent.
Traditionally, the representational capacity of neural networks, both in terms of lower bounds (what they can provably do) as well as upper bounds (what they can provably \emph{not} do), has been studied in terms of Boolean sequential models of computation, such as finite-state automata and Turing machines \citep[e.g.,][]{Kleene1956,Minsky1954,Siegelmann1992OnTC,hao-etal-2018-context,merrill-2019-sequential,merrill-etal-2020-formal,merrill-etal-2022-saturated,merrill2022extracting}.
Recent work has extended this paradigm to work with probabilistic models of computation \citep{svete-cotterell-2023-recurrent,nowak-etal-2023-representational}, but so far only for LMs based on recurrent neural networks.\looseness=-1

However, the sequential nature of classical models makes the connection to the inherently \emph{parallelizable} transformer architecture less straightforward and has resulted in a number of results upper-bounding their representational capacity \citep{hahn-2020-theoretical,bhattamishra-etal-2020-on-ability,chiang-cholak-2022-overcoming,hao-etal-2022-formal,merrill-sabharwal-2023-parallelism}.
We connect transformer LMs to a classical class of LMs that lend themselves particularly well to parallelized computations: \ngram LMs.
We show that both hard as well as sparse attention transformer LMs can represent any \ngram LM (\cref{thm:transformers-n-gram-label,thm:transformers-n-gram-label-sparse}).\footnote{An analysis completely analogous to practical implementations would also consider \emph{soft} attention transformer LMs, whose full support when attending over the preceding symbols makes the analysis trickier. We, therefore, omit its analysis here and reserve it for a separate treatment.}
This gives us a concrete lower bound on their probabilistic representational capacity.
We also study the role of the number of heads (\cref{thm:transformers-n-gram-label}) and the number of layers (\cref{thm:transformers-n-gram-label-multi-layer}), illustrating a trade-off between the number of heads, layers, and the complexity of the non-linear transformations required for the simulation of \ngram LMs.
Altogether, these results offer a step towards understanding the probabilistic representational capacity of transformer LMs and the mechanisms they might employ to implement formal models of computation.\looseness=-1

\section{Preliminaries}
Let $\alphabet$ be an alphabet, i.e., a finite, non-empty set of symbols, and $\kleene{\alphabet}$ the (infinite) set of all strings formed from symbols of $\alphabet$.
Most modern LMs define $\pLM\left(\str\right)$ for $\str \in \kleene{\alphabet}$ autoregressively---as a product of conditional probability distributions:
\begin{equation} \label{eq:lnlm}
    \pLM\left(\str\right) \defeq \pLNSM\left(\eos\mid\str\right) \prod_{\tstep = 1}^{|\str|} \pLNSM\left(\symt \mid \strlt\right).
\end{equation}
Here, $\eos \notin \alphabet$ is a distinguished \underline{e}nd-\underline{o}f-\underline{s}tring symbol.
The $\eos$ symbol enables us to define the probability of a string purely based on the conditional distributions.
Such a factorization can be done without loss of generality \citep{du-etal-2023-measure}.
We define $\eosalphabet \defeq \alphabet \cup \left\{\eos\right\}$.
Further, the conditional probability distributions $\pLNSM\left(\symt \mid \strlt\right)$ are usually defined based on \emph{vectorial} representations of $\strlt$ computed by some function $\enc\colon \kleene{\alphabet} \to \R^\hiddDim$.
This leads us to the definition of representation-based LMs below.
\begin{definition} \label{def:repr-lm}
    Let $\alphabet$ be an alphabet and $\enc\colon \kleene{\alphabet} \to \R^\hiddDim$ a \defn{representation function} encoding strings as $\hiddDim$-dimensional representations.
    Let $\outMtx \in \R^{|\eosalphabet| \times \hiddDim}$ be an \defn{output matrix}.
    A \defn{representation-based} LM $\pLM$ defines the conditional probability distributions $\pLNSM\left(\symt \mid \strlt\right)$ as\footnote{One could, more generally, swap the softmax for any other normalization function, such as the sparsemax \citep{sparsemax}. Here, however, we focus on the softmax for conciseness.} 
    \begin{equation}
        \pLNSM\left(\symt \mid \strlt\right) \defeq \softmaxfunc{\outMtx \; \enc\left(\strlt\right)}{\symt}.
    \end{equation}
\end{definition}

At a high level, we are interested in encoding an arbitrary \ngram LM using a transformer LM.
To do so, we need a notion of equivalence between language models.
In this paper, we will work with the following simple definition.
\begin{definition} \label{def:weak-equivalence}
    Two LMs $\pLM$ and $\qLM$ over $\kleene{\alphabet}$ are \defn{weakly equivalent} if $\pLM\left(\str\right) = \qLM\left(\str\right)$ for all $\str \in \kleene{\alphabet}$.
\end{definition}
This paper precisely explains and proves the following theorem, stated informally below.
\begin{theorem}[Informal]
\label{thm:informal}
For every \ngram LM, there exists a weakly equivalent (\{hard, sparse\} attention) transformer LM.
\end{theorem}

\subsection{An Aside about Boolean Recognition}
Fundamentally, \Cref{thm:informal} is about weak equivalence (\Cref{def:weak-equivalence}) between two LMs.
In this subsection, we make our case against treating LMs as \emph{recognizers}.
The most common manner of analyzing a language model as a recognizer is based on using its representations as an input to a classifier \citep{merrill-2019-sequential,merrill-etal-2020-formal}.
We recapitulate a common definition below.
\begin{definition} \label{def:classifier}
Let $\pLM$ be a representation-based LM with the representation function $\enc\colon \kleene{\alphabet} \to \R^\hiddDim$ and let $\funcg\colon \R^\hiddDim \to \set{0, 1}$ be a classifier.
The \defn{binary language} of $\pLM$ with $\funcg$ is defined as
\begin{equation}\label{eq:eps-support}
  \lang_\funcg\left(\pLM\right) \defeq  \set{ \str \in \kleene{\alphabet} \mid \funcg\left(\enc\left(\str\right)\right) = 1 }.
\end{equation}
\end{definition}
Related is the notion of truncated recognition.
\begin{definition}[\citet{hewitt-etal-2020-rnns}, Definition 4] \label{def:truncated}
Let $\pLM$ be a language model over $\kleene{\alphabet}$ and $\alpha > 0$.
The $\alpha$-\defn{truncated language} of $\pLM$ is defined as
\begin{align}
  \lang_\alpha\left(\pLM\right) \defeq \{ \str &\in \kleene{\alphabet} \mid \pLM\left(\eos \mid \str\right) \geq \alpha \\
  & \text{and } \pLM\left(\symt \mid \strlt\right) \geq \alpha \quad \forall \tstep \in \NTo{|\str|} \}. \nonumber
\end{align}
\end{definition}
There are many results in the literature treating transformers' ability to recognize languages in the sense of the two definitions above.
For instance, transformers are unable to recognize the Dyck language with more than one bracket type and the $\textsc{Parity}$ language in the sense of \cref{def:classifier} \citep{hahn-2020-theoretical}, but \emph{can} recognize \emph{bounded} Dyck languages in the sense of \cref{def:truncated} \citep{yao-etal-2021-self}.
Our indictment of analyzing $\lang_\funcg\left(\pLM\right)$ and $\lang_\alpha\left(\pLM\right)$ is that proceeding in such a manner disregards the probabilities assigned to strings by $\pLM$, which we view as essential to language modeling.
Moreover, \cref{def:classifier} depends on the form of the classifier $\funcg$ while \Cref{def:truncated} depends on the hyperparameter $\alpha$.
For example, positively classified strings from a language could have their conditional probabilities only slightly above the classification threshold and the negatively classified ones only slightly below the threshold \citep{hahn-2020-theoretical}, which hides the true distribution defined by the LM.
In this context, our contention is that $\lang_\funcg\left(\pLM\right)$ and $\lang_\alpha\left(\pLM\right)$ are not useful definitions for studying the representational capacity of LMs. 
Instead, we advocate for analyzing LMs as \emph{probabilistic} formal languages.

\subsection{Language Modeling with \texorpdfstring{\ngram{}}{n-gram}s}
The next-symbol probabilities in \ngram LMs are computed under the \ngram assumption.
\begin{assumption} \label{def:ngram}
    The \defn{\ngram{} assumption} states that the conditional probability of the symbol $\sym_\tstep$ given $\strlt$ only depends on $\ngr-1$ previous symbols $\str^{\tstep - 1}_{\tstep - \ngr + 1} \defeq \sym_{\tstep-1},\ldots,\sym_{\tstep-\ngr+1}$:\looseness=-1
    \begin{equation} \label{eq:n-gram-assumption}
        \pLNSM\left(\sym_\tstep\mid \str_{<\tstep}\right) = \pLNSM\left(\sym_\tstep \mid \str^{\tstep - 1}_{\tstep - \ngr + 1}\right).
    \end{equation}
\noindent We will refer to $\str^{\tstep - 1}_{\tstep - \ngr + 1}$ as the \defn{history} of $\sym_\tstep$.
\end{assumption}

\begin{figure}
    \centering
    \begin{tikzpicture}
        \node[align=center] at (0, 0.6) {The \quad quick \quad brown \quad fox \quad jumps \quad over \ldots};

        \draw[draw=none, rounded corners, fill=DarkBlue!15] (-4, -0.3) rectangle (0.9, 0.3);
        \draw[draw=none, rounded corners, fill=DarkBlue!30] (-0.5 + 0.45, -0.3) rectangle (0.9, 0.3);
        \draw[draw=none, rounded corners, fill=AccentBlue!15] (-
        3.1, -0.3-0.61) rectangle (2.3, 0.3-0.61);
        \draw[draw=none, rounded corners, fill=AccentBlue!30] (0.9, -0.3-0.61) rectangle (2.3, 0.3-0.61);
        \draw[draw=none, rounded corners, fill=LightBlue!15] (-1.6, -0.3-0.61-0.61) rectangle (3.5, 0.3-0.61-0.61);
        \draw[draw=none, rounded corners, fill=LightBlue!30] (2.3, -0.3-0.61-0.61) rectangle (3.5, 0.3-0.61-0.61);

        \node[align=center] at (-2, 0) {$\pLM\left(\text{\small fox} \mid \text{\small The quick brown}\right)$};
        \node[align=center] at (-1.05, -0.61) {$\pLM\left(\text{\small jumps} \mid \text{\small quick brown fox}\right)$};
        \node[align=center] at (0.35, -0.61-0.61) {$\pLM\left(\text{\small over} \mid \text{\small brown fox jumps}\right)$};
        \node[align=center] at (0.45, 0) {$\cdot$};
        \node[align=center] at (1.6, -0.61) {$\cdot$};
        \node[align=center] at (2.9, -0.61-0.61) {$\cdot$};
    \end{tikzpicture}
    \caption{An illustration of how a $4$-gram LM computes the probability of a string.
        All conditional probabilities can be computed in parallel and then multiplied into the probability of the entire string.}
    \label{fig:example-ngram}
\end{figure}

\paragraph{Padding.}
\cref{eq:n-gram-assumption} assumes the existence of $\ngr - 1$ preceding symbols that define the conditional distribution $\pLNSM\left(\sym_\tstep\mid \str_{<\tstep}\right)$.
To ensure this is the case even at the beginning of the string, it is standard to \emph{pad} the input string with $\ngr - 1$ \underline{b}eginning-\underline{o}f-\underline{s}tring symbols $\bos$.
For ease of notation, we index the $\ngr - 1$ $\bos$ tokens with indices $-\ngr + 2, \ldots, 0$ so that \ngram LMs conveniently fit the autoregressive factorization from \cref{eq:lnlm}.
We also define $\bosalphabet \defeq \alphabet \cup \set{\bos}$.\looseness=-1

Despite their simplicity, \ngram LMs have a storied place in language modeling \citep{6773024,1162650,10.5555/907280,1454428,4767370,10.5555/108235.108270,NIPS2000_728f206c,10.5555/944919.944966,Bengio2006,SCHWENK2007492,heafield-2011-kenlm,heafield-etal-2013-scalable}.
Because the conditional probabilities of \ngram LMs only depend on the previous $\ngr - 1$ symbols, different parts of the string can be processed independently, i.e., in parallel.
This facilitates a natural connection to transformer LMs since parallelizability is a prevalent feature of the architecture and one of its main advantages over other neural LMs such as RNN LMs \citep{Vaswani2017}.\looseness=-1

\subsection{Transformer Language Models}
Transformer LMs are LMs whose conditional distributions $\pLNSM\left(\symt \mid \strlt\right)$ are computed by a \defn{transformer}.
A transformer is a composition of multiple transformer \defn{layers}, each of which implements the \defn{attention mechanism}.
We give definitions of these building blocks in what follows.\looseness=-1

\begin{table*} \centering\footnotesize
    \begin{tabular}{@{}clp{8cm}@{}}
        \toprule
        Symbol & Type & Meaning \\
        \midrule
        $\NTo{N}$ & $\subset \N$ & The set $\set{1, \ldots, N}$ for $N \in \N$. \\
         $\sym$ & $\in \alphabet$ & A symbol, element of $\alphabet$. \\
        $\alphabet, \bosalphabet, \eosalphabet$ & alphabet & $\alphabet$ is a set of symbols, $\bosalphabet \defeq \alphabet \cup \set{\bos}$, $\eosalphabet \defeq \alphabet \cup \set{\eos}$ \\
        $\str$ & $\in \kleene{\alphabet}$ & A string over $\alphabet$. \\
        $\str^{\idxi}_{\idxj}$ & $\in \kleene{\alphabet}$ & A substring of $\str$, a string. \\
        $\onehot{\sym}$ & $\in \{0, 1\}^{\nsymbols}$ & One-hot encoding of the symbol $\sym \in \alphabet$. \\
        $\hiddDim$ & $\in \N$ & Size of the contextual representations in the transformer. \\
        $\simplexFun{N - 1}$ & $\subseteq \R^{N}$ & The $N-1$-dimensional probability simplex. \\ 
        $\tfscorefun$ & $\R^{\hiddDim} \times \R^{\hiddDim} \to \R$ & A scoring function. \\
        $\projfunc$ & $\R^N \to \simplexFun{N - 1}$ & A normalization function. \\
        $\qTransf$, $\kTransf$, $\vTransf$, $\oTransf$ & $\R^\hiddDim \to \R^\hiddDim$ & The query, key, value, and output functions. \\
        $\fTransf$ & $\R^\hiddDim \to \R^\hiddDim$ & The final transformer LM transformation function. \\
        $\enc$ & $\kleene{\alphabet} \to \R^\hiddDim$ & The string representation function. \\
        $\posInEmbedding$ & $\bosalphabet \times \N \to \R^\hiddDim$ & The position-augmented representation function. \\
        $\tfnumlayer$ & $\in \N$ & Number of layers. \\
        $\tfheadnum$ & $\in \N$ & Number of heads. \\
        $\tfheadCombine$ & $\R^{\tfheadnum \hiddDim} \to \R^\hiddDim$ & The head combining function. \\
        $\left(\cdot\;; \cdots; \cdot\right)$ & & Vertical concatenation operator of vectors or matrices. \\
        \bottomrule
    \end{tabular}
    \caption{A summary of the notation used in the paper.}
    \label{tab:notation}
\end{table*}

\paragraph{Notation.}
We use bold, unitalicized letters such as $\vx \in \R^\hiddDim$ to denote real-valued vectors and italicized letters $\evx_\idxj \in \R$ for their entries.
Capital bold letters such as $\mX \in \R^{N \times \hiddDim}$ denote matrices.
All vectors are \emph{column} vectors unless transposed.
We define the vertically stacking operator $\left(\cdot \; ; \cdots; \; \cdot\right)$, which denotes the vertical concatenation of the $\hiddDim$-dimensional \emph{column} vectors $\vx_1, \ldots, \vx_N$ into a $N\hiddDim$-dimensional vector $\left(\vx_1; \cdots; \vx_N\right) \in \R^{N \hiddDim}$ and the concatenation of the $\hiddDim$-dimensional \emph{row} vectors $\vx^\top_1, \ldots, \vx^\top_N$ into a matrix $\mX \in \R^{N \times \hiddDim}$ with $N$ rows and $\hiddDim$ columns.
Given the matrix $\mX = \left(\vx^\top_1; \cdots; \vx^\top_N\right)$, we write $\mX_\idxn = \left(\vx^\top_1; \cdots; \vx^\top_\idxn\right)$ for the submatrix composed of the first $n$ rows.
We call a function $\tfscorefun \colon \R^{\hiddDim} \times \R^{\hiddDim} \to \R$ whose purpose is to evaluate the compatibility of two vectors a \defn{scoring function}.
A \defn{normalization function} $\projfunc\colon \R^N \to \simplexFun{N - 1}$ maps vectors in $\R^N$ to $N$ probabilities.
Here, $\simplexFun{N - 1} \defeq \set{\vx \in \left[0, 1\right]^N \mid \sum_{\idxn = 1}^N \evx_\idxn = 1}$ is the $(N-1)$-dimensional probability simplex.
This notation is summarized in \cref{tab:notation}.

\paragraph{The Attention Mechanism.}
The attention mechanism works as follows.
It takes a \defn{\underline{q}uery} vector $\vq \in \R^{\hiddDim}$ and two matrices: The matrix $\mK \in \R^{N \times \hiddDim}$ of \defn{\underline{k}eys} and the matrix $\mV\in \R^{N \times \hiddDim}$ of \defn{\underline{v}alues} and computes a weighted average of the value vectors based on the compatibilities of the key vectors and the query vector, as determined by a scoring function $\tfscorefun$.
A formal definition is given below.
\begin{definition}[Attention Mechanism] \label{def:attention}
    Let $\tfscorefun$ be a scoring function and $\projfunc$ a normalization function.
    Let $\vq \in \R^{\hiddDim}$ be a query vector and let $\mK = \left(\vk^\top_1; \cdots; \vk^\top_N\right) \in \R^{N \times \hiddDim}$ and $\mV = \left(\vv^\top_1; \cdots; \vv^\top_N\right) \in \R^{N \times \hiddDim}$ be matrices of keys and values, respectively.
    An \defn{attention mechanism} $\attn\colon \R^\hiddDim \times \R^{N \times \hiddDim} \times \R^{N \times \hiddDim} \to \R^\hiddDim$ is defined as
    \begin{equation} \label{eq:attention-sum}
        \attn\left(\vq, \mK, \mV\right) \defeq \sum_{\idxn = 1}^{N} \evs_\idxn\vv_\idxn,
    \end{equation}
    where 
    \begin{equation}
        \vs \defeq \projfunc\left(\tfscorefun\left(\vq, \vk_1\right), \dots ,\tfscorefun\left(\vq, \vk_N\right) \right)
    \end{equation}
    is the vector of normalized scores between the query $\vq$ and the keys in $\mK$.
\end{definition}
The most standard implementation of the scoring function $\tfscorefun$ is the (scaled) inner product $\tfscorefun\left(\vq, \vk\right) \defeq \innerProd{\vq}{\vk}$.
Some of our results rely on this standard formulation.
However, some also rely on the more general, but still simple and just as efficiently computable scoring functions.

\paragraph{The Transformer Architecture.}
A transformer layer uses the attention mechanism to compute augmented representations $\vz_\tstep = \attn\left(\vq_\tstep, \mK_\tstep, \mV_\tstep\right)$ of the input representations $\mX_\tstep = \left(\vx_1; \cdots; \vx_\tstep\right)$.
The query $\vq_\tstep$, the keys $\mK_\tstep$, and values $\mV_\tstep$ are all transformations of the input representations $\mX_\tstep$.

\begin{definition} \label{def:transformer-layer}
Let $\qTransf, \kTransf, \vTransf, \oTransf\colon \R^\hiddDim\!\to\!\R^\hiddDim$ be the query, key, value, and \defn{\underline{o}utput} functions.
A \defn{transformer layer} is a function $\tflayer\colon\R^{\finaltstep \times \hiddDim} \to \R^{\finaltstep \times \hiddDim}$ that computes 
\begin{equation}
    \tflayer\left(\vx_1^\top; \ldots; \vx_\strlen^\top\right) = \left(\vz_1^{\top}; \ldots;  \vz_\finaltstep^{\top}\right) \in \R^{\finaltstep \times \hiddDim}
\end{equation}
for $\tstep \in \NTo{\finaltstep}$ where
\begin{subequations}
    \begin{alignat}{2} \label{eq:attn-block-1}
        \va_\tstep &\defeq \attn\left(\vq_\tstep, \mK_\tstep, \mV_\tstep\right) + \tflayerinputsy_\tstep &&\in \R^\hiddDim \\
        \tflayeroutputsy_\tstep &\defeq \oTransf\left(\va_\tstep\right) + \va_\tstep &&\in \R^\hiddDim. \label{eq:attn-block-2}
    \end{alignat}
\end{subequations}
Here, we define 
\begin{subequations}
    \begin{alignat}{2}
        \vq_\tstep &\defeq \qTransf\left(\vx_\tstep\right) &&\in \R^\hiddDim \\
        \mK_\tstep &\defeq \left(\kTransf\left(\vx_1\right)^\top; \cdots; \kTransf\left(\vx_\tstep\right)^\top\right) &&\in \R^{\tstep \times \hiddDim} \\
        \mV_\tstep &\defeq \left(\vTransf\left(\vx_1\right)^\top; \cdots; \kTransf\left(\vx_\tstep\right)^\top\right) &&\in \R^{\tstep \times \hiddDim}.
    \end{alignat}
    \end{subequations}
\emph{Note:} For simplicity, we do not include layer normalization. 
\end{definition}

Without further modification, the transformations applied by the transformer layer are position-invariant, which necessitates the addition of explicit positional information.
\begin{definition} \label{def:positional-encodings}
    A position-augmented symbol \defn{representation function} $\posInEmbedding\colon \alphabet \times \N \to \R^\hiddDim$ 
    is a function representing symbols and their positions as $\hiddDim$-dimensional vectors.
\end{definition}
Position-augmented symbol representation functions are often implemented as an addition or concatenation of separate symbol-only and position-only representation functions \citep{Vaswani2017}.
Here, we define it more generally as any function of the symbol and its position.
\begin{definition} \label{def:static-representations}
    A \defn{static encoding} $\staticRepr$ is a function $\staticRepr\colon \alphabet^\strlen \to \R^{\strlen \times \hiddDim}$ defined for any $\strlen \in \N$ as
    \begin{equation}
        \staticRepr\left(\str\right) \defeq \left(\posInEmbeddingFun{\symone, 1}^\top; \cdots; \posInEmbeddingFun{\sym_\strlen, \strlen}^\top\right).
    \end{equation}
\end{definition}

Multiple transformer layers are stacked into a transformer, which computes the (deep) contextual representations of all symbols in the string.
\begin{definition} \label{def:transformer}
    For $\tfnumlayer \in \N$, let $\tflayer_\tflayeridx$ for $\tflayeridx \in \NTo{\tfnumlayer}$ be transformer layers.
    Let $\staticRepr$ be a static encoding.
    An $\tfnumlayer$-layer \defn{transformer} $\tf$ is defined as
    \begin{equation} \label{eq:transformer-model}
        \tfFun{\staticRepr} \defeq \tflayer_\tfnumlayer \circ \cdots \circ \tflayer_1 \circ \staticRepr.
    \end{equation}
\end{definition}
A transformer computes the contextual representations of the symbols $\str = \sym_1 \cdots \sym_\strlen$ as 
\begin{equation}
    \left(\vx_1^{\tfnumlayer\top}; \cdots; \vx_\strlen^{\tfnumlayer\top}\right) \defeq \tf\left(\staticRepr\right)\left(\str\right).
\end{equation}
If $\staticRepr$ is clear from the context or arbitrary, we will omit it as an argument to $\tf$ and just write $\tf\left(\str\right)$.
\begin{definition} \label{def:enc}
    Let $\tf$ be a transformer, $\fTransf\colon \R^{\hiddDim} \to \R^{\hiddDim}$ the \underline{f}inal representation transformation function, and $\str \in \kleene{\alphabet}$ with $|\str| = \strlen$.
    We define\looseness=-1
    \begin{equation} \label{eq:enc}
        \tfencfun\left(\str\right) \defeq \fTransf\left(\vx_{\strlen}^\tfnumlayer\right), 
    \end{equation}
    where $\vx_{\strlen}^\tfnumlayer$ is the representation of the $\strlen\textsuperscript{th}$ symbol in $\str$ computed by $\tf$, i.e., $\left(\vx_1^{\tfnumlayer\top}; \cdots; \vx_\strlen^{\tfnumlayer\top}\right) = \tf\left(\str\right)$.
\end{definition}

\paragraph{Transformer Language Models.}
So far, we have only defined how the transformer architecture can be used to compute the contextual representations of the symbols.
To complete the definition, we define a transformer \emph{language model} as follows.
\begin{definition} \label{def:transformer-plnsm}
    A \defn{transformer LM} $\tfpLM$ is the representation-based autoregressive LM with the representation function $\tfencfun$ from \cref{eq:enc}.
    That is, $\tfpLM$ defines the conditional probability distributions
    \begin{align} \label{eq:transformer-plnsm}
        \tfpLM\left(\symt \mid \strlt \right) \defeq \softmaxfunc{\embedMtx\,\tfencfun\left(\strlt\right)}{\symt}.
    \end{align}
\end{definition}

\subsubsection{Variants of the Attention Mechanism}
In this subsection, we discuss many common variants of the attention mechanism.
First, \defn{multi-headed} attention uses $\tfheadnum$ \defn{attention heads} to compute $\tfheadnum$ representations of the symbols in the string.
The representations constructed by the different attention heads are concatenated into a long vector and projected down to the output size of a single head with a head-combiner function $\tfheadCombine$.
\begin{definition}
    For $\tfnumlayer, \tfheadnum \in \N$, let $\tflayer^h_\tflayeridx\colon\R^{\strlen \times \hiddDim} \to \R^{\strlen \times \hiddDim}, \tflayeridx \in \NTo{\tfnumlayer}, h \in \NTo{\tfheadnum}$ be transformer layers.
    Define $\tflayer_\ell \colon \R^{\strlen \times \hiddDim} \to \R^{\strlen \times \tfheadnum \hiddDim}$ as
    \begin{equation}
        \tflayer_\ell\left(\mX\right) \defeq \left(
            \tflayer^1_\tflayeridx\left(\mX\right)^\top; \cdots ; \tflayer^{\scaleto{\tfheadnum}{5pt}}_\tflayeridx\left(\mX\right)^\top
        \right)^\top.
    \end{equation}
    Furthermore, let $\tfheadCombine\colon \R^{\tfheadnum\hiddDim} \to \R^\hiddDim$.
    An $\tfnumlayer$-layer transformer with $\tfheadnum$ heads computes:
    \begin{align}
        \tfFun{\staticRepr} \defeq \tflayer_\tfnumlayer \circ \tfheadCombine \circ \cdots \circ \tfheadCombine \circ \tflayer_1 \circ \staticRepr,
    \end{align}
    where $\tfheadCombine$ is applied \emph{row-wise} to project the representations of $\tfheadnum$ heads to $\R^\hiddDim$.
\end{definition}

\paragraph{Attention types.}
Attention weights are computed by normalizing the scores $\tfscorefun\left(\vq, \vk_1\right), \dots ,\tfscorefun\left(\vq, \vk_\tstep\right)$.
The choice of the projection function $\projfunc$ determines the type of attention and has concrete implications on representational capacity \citep{hao-etal-2022-formal}.
\begin{definition}  \label{def:hard-attention}
    \defn{Hard attention} is computed with the $\hardmaxAvg$ projection function:
    \begin{equation}
        \hardmaxAvg\left(\vx\right)_\idxd \defeq \begin{cases}
            \frac{1}{m} &\ifcondition \idxd\in \argmax\left(\vx\right) \\
            0 &\otherwisecondition
        \end{cases}
    \end{equation}
    for $\idxd \in \NTo{\hiddDim}$, where $\vx \in \R^\hiddDim$ and $m \defeq |\argmax\left(\vx\right)|$ is the cardinality of the argmax set.
\end{definition}
We also introduce \emph{sparse} attention, which uses the sparsemax normalization function to compute the attention weights.
\begin{definition}  \label{def:sparse-attention}
    \defn{Sparse attention} is computed with the $\sparsemax$ projection function:
    \begin{equation}\label{eq:spmax}
        \sparsemaxfunc{\vx}{} \defeq \argmin_{\rvp\in \Simplexdminus} \norm{\rvp- \vx}^2_2.
    \end{equation}
\end{definition}

\section{Hard Attention Transformer LMs} \label{sec:hard-attention}
This section presents a set of results describing the representational capacity of hard attention transformer LMs.
Concretely, we show that transformer LMs with hard attention can represent \ngram LMs, either using $\ngr - 1$ heads (\cref{thm:transformers-n-gram-label}) or $\ngr - 1$ layers (\cref{thm:transformers-n-gram-label-multi-layer}).
Simulation is possible even with a single head and a single layer (\cref{thm:transformers-n-gram-single}) but might require a more elaborate set of non-linear transformations and positional encodings whose precision scales linearly with the string length.

\begin{restatable}{reTheorem}{singleLayerKHeadsTheorem} \label{thm:transformers-n-gram-label}
    For any \ngram LM, there exists a weakly equivalent single-layer hard attention transformer LM with $\ngr - 1$ heads.
\end{restatable}
\begin{proof}[Proof intuition]
    Given an \ngram LM $\pLM$, we can construct a weakly equivalent LM $\tfpLM$ defined by a transformer $\tf$ that looks back at the preceding $\ngr - 1$ positions using $\ngr - 1$ heads, each of them uniquely attending to exactly one position.
    The symbols attended to can be used to identify the full history, which can be used to access the conditional distribution over the next symbol.
    This is illustrated in \cref{fig:transformer-n-gram-label}.
    See \cref{sec:proofs-hard-attention} for the full proof.
\end{proof}

\cref{thm:transformers-n-gram-label} shows that transformer LMs with hard attention can represent \ngram LMs, establishing, to the best of our knowledge, the first concrete relationship between transformer LMs and probabilistic languages.
A natural follow-up question then is whether $\ngr - 1$ heads are \emph{necessary} to correctly simulate an \ngram LM.
Besides aiming to illuminate different mechanisms enabling the implementation of classical LMs, this question also follows the line of inquiry about the \emph{uniqueness} and \emph{interpretability} of the representations of formal models by neural LMs \citep{liu2023transformers}.
The following two theorems show that the intuitive construction using $\ngr - 1$ heads is far from unique: \cref{thm:transformers-n-gram-label-multi-layer} shows that a similarly simple simulation is possible with $\ngr - 1$ layers and a single head, while \cref{thm:transformers-n-gram-single} shows that even a transformer LM with a single head and a single layer can simulate an \ngram LM, albeit with more complex position invariant transformation $\fTransf$.
This suggests that there is no canonical way of determining whether a transformer LM has learned an \ngram LM by looking at individual components (e.g., positions attended to by the different heads).\looseness=-1

\begin{restatable}{reTheorem}{kLayersSingleHeadTheorem}  \label{thm:transformers-n-gram-label-multi-layer}
    For any \ngram LM, there exists a weakly equivalent $(\ngr - 1)$-layer hard attention transformer LM with a single head.
\end{restatable}
\begin{proof}[Proof intuition]
    Whereas the transformer LM constructed in \cref{thm:transformers-n-gram-label} used $\ngr - 1$ heads to look at all the $\ngr - 1$ positions of interest, an $\ngr - 1$-layer transformer LM can use the $\ngr - 1$ layers to look back at the \emph{immediately} preceding position and copy it forward $\ngr - 1$ times (keeping the current symbol there as well).
    After $\ngr - 1$ layers of such transformations, the entire history can be read from the current contextual representation.
    See \cref{sec:proofs-hard-attention} for the full proof.
\end{proof}

Apart from using hard attention, both transformer LMs used in \cref{thm:transformers-n-gram-label,thm:transformers-n-gram-label-multi-layer} rely on modeling assumptions often found in practical implementations of the transformer: The transformations $\qTransf, \kTransf$, and $\vTransf$ are linear functions, the scoring function is implemented as a dot-product and positional encodings are bounded.
This makes the results comparable to practical implementations. 
The following theorem, in contrast, shows that if we permit the use of less standard components, transformer LMs can identify the history of interest using only a single head and a single layer.
\begin{restatable}{reTheorem}{singleLayerSingleHeadTheorem}  \label{thm:transformers-n-gram-single}
    For any \ngram LM, there exists a weakly equivalent single-layer hard attention transformer LM with a single head.
\end{restatable}
\begin{proof}[Proof intuition]
    The bulk of this construction lies in the encoding $\str^{\tstep - 1}_{\tstep - \ngr + 1}$ in a vector that can be constructed by a single attention head in one layer.
    This is done by an attention head that 
    \begin{enumerate*}[label=\textit{(\arabic*)}]
        \item puts non-zero attention on only the previous $\ngr - 1$ symbols and
        \item encodes the identities and the positions of symbols in a $\bosnsymbols$-dimensional value vector.
    \end{enumerate*}
    The value vector can then be decoded into a one-hot encoding of $\str^{\tstep - 1}_{\tstep - \ngr + 1}$ by an $\ngr - 1$-layer MLP that defines $\fTransf$, which allows us to match the conditional probabilities of the \ngram LM as in \cref{thm:transformers-n-gram-label,thm:transformers-n-gram-label-multi-layer}.
    See \cref{sec:proofs-hard-attention} for the full proof.
\end{proof}

\section{Sparse Attention Transformer LMs} \label{sec:sparse-attention}
While the results in \cref{sec:hard-attention} concretely characterize the abilities of hard attention transformer LMs, the assumption of hard attention is somewhat removed from practical implementations of the model. 
Those most often rely on differentiable normalization functions, such as the softmax.\footnote{As noted in \cref{sec:intro}, the analysis of soft attention transformers requires a different type of analysis in terms of approximation of the probabilities. 
A complete study would have to consider the approximation over \emph{arbitrarily} long strings (since $\kleene{\alphabet}$ is an infinite set), which is difficult by simply scaling model parameters to a large constant.
We focus on exact simulation here, but conjecture that soft attention transformers can approximate LMs whose \emph{average} string length is finite.}
However, the full support of the softmax function makes the connection to formal models of computation difficult \citep{hahn-2020-theoretical}.
To bring the theoretical models closer to practical implementations yet still be able to make clear analogies to formal models of computation, we now consider \emph{sparse} attention transformers, which use the sparsemax normalization function. 
The sparsity allows sparse attention transformers to simulate \ngram LMs just like hard attention transformers while relying on differentiable operations.

\begin{restatable}{reTheorem}{singleLayerSparseTheorem} \label{thm:transformers-n-gram-label-sparse}
    For any \ngram LM, there exists a weakly equivalent single-layer sparse attention transformer LM with $\ngr - 1$ heads.
\end{restatable}
\begin{proof}[Proof intuition]
The intuition behind the simulation with sparse attention is similar to the hard attention one; each head attends to a single position, as illustrated in \cref{fig:transformer-n-gram-label}.
Effectively, the construction results in a sparse attention transformer that simulates hard attention.
In contrast to \cref{thm:transformers-n-gram-label,thm:transformers-n-gram-label-multi-layer}, we here require a model with \emph{unbounded} positional encodings and a non-linearly transformed dot-product scoring function.
Intuitively, the unbounded positional encodings are required to scale the unnormalized attention scores to differ enough for the sparsemax to focus on a single position.
The rest of the proof follows that of \cref{thm:transformers-n-gram-label}; see \cref{sec:proofs-sparse-attention} for the details.
\end{proof}

\cref{thm:transformers-n-gram-label-sparse} (representing an \ngram LM with $\ngr - 1$ heads) could naturally be extended to analogs of \cref{thm:transformers-n-gram-label-multi-layer} (representing an \ngram LM with $\ngr - 1$ layers) and \cref{thm:transformers-n-gram-single} (representing an \ngram LM with a single head and a single layer) using a similar adaptation of the construction from the hard attention case to the sparse attention one as in \cref{thm:transformers-n-gram-label-sparse}.

\section{Space Complexity}
In \cref{sec:hard-attention} and \cref{sec:sparse-attention}, we describe lower bounds that tell us what types of probability distributions transformer LMs \emph{can} represent, but do not say how \emph{efficiently} they can do so.
The space complexity of simulating \ngram LMs is discussed in this section.
We focus on hard-attention transformer LMs with multiple heads or multiple layers (\cref{thm:transformers-n-gram-label,thm:transformers-n-gram-label-multi-layer}) since their modeling assumptions (bar hard attention) are closest to practical implementations.
The constructive proofs of \cref{thm:transformers-n-gram-label,thm:transformers-n-gram-label-multi-layer} allow us to directly analyze the space requirements for the simulation of \ngram LMs, both in terms of 
\begin{enumerate*}[label=\textit{(\arabic*)}]
    \item the size of the contextual representations $\mX^h_\ell$ as well as 
    \item the number of bits required to represent the individual entries of the vectors $\vx^h_{\ell, \tstep}$.
\end{enumerate*}

\subsection{Scaling with Respect to the Number of Computational Steps} \label{sec:bits-scaling}
We first address the second point.
Specifically, we are interested in how the number of bits scales with respect to $\tstep$, the number of computational steps performed during the generation of a string $\str \in \kleene{\alphabet}$.
As summarized by \cref{tab:summary}, the models constructed in the proofs of \cref{thm:transformers-n-gram-label,thm:transformers-n-gram-label-multi-layer} use positional encodings with entries of the form $\sqrt{\frac{1}{\tstep}}$ for $\tstep \in \N$, i.e., they contain square roots of rational numbers.
This makes the scaling of the space complexity difficult, as square roots of rational numbers are not in general representable with a finite number of bits.
While this might seem discouraging, we emphasize that these specific positional encodings were only used to keep the contextual representations bounded and the scoring function in line with the original formulation \citep{Vaswani2017} and concurrent work \citep{merrill2023expressive}.
A closer look at the constructions in the proof of \cref{thm:transformers-n-gram-label,thm:transformers-n-gram-label-multi-layer,thm:transformers-n-gram-label-sparse} reveals that simpler (but unbounded) positional encodings with a less standard scoring function can be used to the same effect.
In particular, we can use positional encodings that contain entries of the form $\tstep$, which only require a \emph{logarithmic} number of bits with respect to $\tstep$.
Since such scaling is required to uniquely encode the positional information in general \citep{merrill-sabharwal-2023-parallelism}, this represents an asymptotically optimal scaling of the space complexity of the contextual representations.\footnote{The construction in \cref{thm:transformers-n-gram-single}, in contrast, relies on encoding the entire preceding string in a single dimension with one digit per position in the string, which requires a number of bits that scales linearly with respect to the string length. A simplification to a logarithmic number of bits does not seem as straightforward.}

Importantly, \ngram LMs are \emph{real-time}: they, by definition, generate a symbol at each step of the computation.
This means that the scaling of the space complexity with respect to the number of computation steps coincides with its scaling with respect to the length of the generated string $\str$---the scaling is logarithmic in $|\str|$.
This is in contrast to non-real-time models of computation which might not generate a symbol at each step of the computation; while those might still require an asymptotically optimal scaling with respect to $\tstep$, the additional computational steps that do not emit any symbol might mean that the space complexity is \emph{unbounded} with respect to the length of the generated string.
An example of such a model is a transformer LM simulating a (probabilistic) Turing machine, which would require the model to not emit symbols at some points of the computation \citep{nowak-etal-2023-representational,nowak-etal-2024-computational}.

\subsection{The Dimensionality of the Contextual Representations}
We now discuss the size of the contextual representations required for the simulation of \ngram LMs.
From a high level, we have to consider two stages: 
\begin{enumerate*}[label=\textit{(\arabic*)}]
    \item the contextual representations $\vx_{\ell, \tstep}^h$ of the different layers and heads and
    \item the size of the final representation $\tfencfun\left(\str\right)$.
\end{enumerate*}
The contextual representations $\vx_{\ell, \tstep}^h$ in stage \textit{(1)} are composed of the symbol and positional encodings.
The symbol representations include (two copies of) the one-hot encodings while the positional encodings include between two and $2 \ngr$ dimensions encoding positional information. 
This means that the per-head and per-layer space complexity scales with $\nsymbols$ and $\ngr$.
For stage \textit{(2)} we use the one-hot encodings of the entire \emph{history} $\str^{\tstep - 1}_{\tstep - \ngr + 1}$, with which we index the matrix of $\bosnsymbols^{\ngr - 1}$ conditional probabilities defined by the \ngram LM. 
While the size of $\tfencfun\left(\str\right)$ is theoretically only lower-bounded by $|\eosalphabet|$ \citep{yang2018breaking,svete-cotterell-2023-recurrent},\footnote{This is because the (logits of the) conditional probabilities can span a $|\eosalphabet|$-dimensional space.} reducing its size requires a lower-rank decomposition of the matrix of conditional probabilities and a corresponding reparametrization of the contextual symbol representation.
This might require a blow-up in the number of bits required to represent individual dimensions, or, in general, result in a real-valued vector that could not be represented on a finite-precision system.
Altogether, most of the space complexity of the contextual representations comes from the one-hot encodings in $\tfencfun\left(\str\right)$, which require $\bosnsymbols^{\ngr - 1}$ dimensions.\looseness=-1

The discussion in this section can be summarized by the following theorem.
\begin{restatable}{reTheorem}{spaceBoundsThm}
    Let $\pLM$ be an \ngram LM over the alphabet $\alphabet$. 
    There exists a weakly equivalent hard-attention transformer LM with contextual representations $\vx$ of size $\bigO{\ngr \nsymbols}$ and representation $\tfencfun\left(\str\right)$ of size $\bigO{|\alphabet|^{\ngr - 1}}$.
    Each of $\vx$'s entries can be represented with $\bigO{\log_2\left(|\str|\right)}$ bits for $\str \in \kleene{\alphabet}$.
\end{restatable}   

\section{Discussion and Related Work} \label{sec:discussion}

To the best of our knowledge, \cref{sec:hard-attention} and \cref{sec:sparse-attention} provide the first results on the probabilistic representational capacity of transformer LMs.

\paragraph{The relevance of \ngram LMs to modern LMs.}
One might rightfully question the utility of connecting the state-of-the-art language modeling architecture to \ngram LMs.
LMs based on the \ngram assumption indeed constitute some of the simplest and least expressive classes of probability distributions.
Nevertheless, \ngram LMs provide a useful playground and theoretical foundation for contextualizing and interpreting the inner workings of modern LMs.
For example, \ngram LMs have been found to comprise a crucial component of \emph{in-context learning}, where attention heads in different layers of the model together identify individual \ngram{}s and base their predictions on the presence of such \ngram{}s \citep{olsson2022context,akyürek2024incontext}.
The presence of specific \ngram{}s might therefore present the foundations of in-context-based knowledge of transformer LMs.
As such, understanding the requirements for correct simulation of \ngram LMs is important for a thorough grasp of the abilities of LMs to learn in context.
Encouraging \ngram{}-LM-based learning has also been observed to aid in-context learning abilities \citep{akyürek2024incontext}.
Existing work has also linked \ngram LMs to other neural network architectures such as one-dimensional convolutional neural networks \citep{merrill-2019-sequential}.
Moreover, the theoretical grounding in classical formal language theory makes precise statements about \ngram LMs possible while their simplicity makes them inherently interpretable and easy to analyze.
\ngram LMs also have a cognitive interpretation \citep{Jager2012-kv}.
Most importantly, however, \ngram LMs lend themselves to \emph{paralellized} processing, which affords a succinct and natural connection to transformer LMs and has been suggested to be a crucial part of any theoretical treatment of transformer LMs \citep{strobl2023transformers,merrill-sabharwal-2023-parallelism}.

\paragraph{Understanding the limitations and abilities of hard attention transformer LMs.}
Our results are the strongest in the hard attention setting, which follows the trend of using hard attention in theoretical treatments.
Hard attention makes singling out the important aspects of the string (in our case, the history) possible.
Concretely, we showcase three different mechanisms that make it possible for transformer LMs to implement \ngram LMs with hard attention and investigate the role of the number of heads and layers.
Particularly, our constructions suggest a possible mechanism in which the different transformer heads or layers can specialize in focusing on different positions in the string, a feature that has previously been suggested as an explanation of how transformer LMs process strings \citep{elhage2021mathematical} and has been observed in trained transformer LMs \citep{olsson2022context,akyürek2024incontext}.\footnote{Note that the constructions presented in this paper are purely meant to showcase the existence of a mechanism that can be used to simulate \ngram LMs; we do not suggest that the same mechanisms will be employed by models used in practice, which is also why do not present any empirical results.
    For example, the use of the sparse one-hot encodings differs from the standard dense representations of symbols.\looseness=-1}
Our presentation of multiple orthogonal mechanisms that can simulate \ngram LMs equivalently well is another confirmation of the observation that algorithmic principles learned or implemented by neural LMs do not always correspond to a \emph{single} intuitive implementation of a formal model of computation, which has concrete implications on interpretability methods for the architecture. 
This has been observed in practical scenarios and warns us that focusing on individual components of the model (that is, using myopic interpretability methods) might result in misleading interpretability results \citep{wen2023transformers}.\looseness=-1

\paragraph{Probabilistic representational capacity.}
We augment existing literature by providing an explicit connection between transformer LMs and \ngram language models.
In line with work comparing transformers to circuits \citep{merrill-sabharwal-2023-parallelism,merrill2023expressive,weiss21a,hao-circuits,merrill2023logic,merrill-etal-2022-saturated,chiang-tighter,angluin2023masked}, we also show the utility of analyzing transformers with parallelizable models of computation, which go hand in hand with the parallelizable nature of the transformer architecture.
Moreover, the formalization of an LM with the output matrix indexed by the contextual representation (cf. \cref{def:transformer-plnsm}) makes it easy to connect existing results on the expressivity of the transformer architecture with the probabilistic setting by defining a mapping from the contextual representation to the conditional probability distribution through the output matrix.
We suppose other simple and parallelizable classes of distributions might lend themselves well to similar probabilistic treatments; probabilistic analysis may be particularly interesting in the context of circuit complexity with sigmoid-activated circuits \citep{185447}, \citep[Chapter 4]{mathNN}.

\paragraph{Connection to (sub-)regular languages}
This work focuses on connecting transformer LMs to \ngram LMs, which are a special instance of the more general class of \emph{sub-regular LMs}.
In words, sub-regular LMs are LMs that can be described without using the full power of the probabilistic finite-state automata \citep{Jager2012-kv}.
In this sense, sub-regular LMs fall below regular languages in the Chomsky hierarchy and themselves define their own hierarchy \citep{Jager2012-kv,heinz-rogers-2013-learning,avcu2017subregular}.
For example, some classes of sub-regular languages do not require sequential processing of the input string that is usually required by finite-state automata for correct recognition.
The intuitive connection between transformer LMs and \ngram LMs encourages further work on the connections between other classes of (sub-)regular LMs and transformer LMs.
Transformers have been linked to (sub)-regular languages by \citet{yao-etal-2021-self} and \citet{liu2023transformers}.
\citet{yao-etal-2021-self} study the ability of transformers to generate bounded hierarchical languages using a transformer but do not extend their analysis to the fully probabilistic case.\footnote{The natural connection of the transformer architecture to both bounded hierarchical languages as well as to \ngram LMs is interesting since those two classes of LMs can both be represented particularly \emph{efficiently} by \emph{recurrent} neural LMs \citep{svete2024theoretical}.}
\citet{liu2023transformers} connect transformers with both hard and soft attention to general finite-state automata (FSAs), which define a strictly larger set of languages than \ngram languages.
This additional generality, however, comes with some caveats. 
\citet[Theorem 1]{liu2023transformers}---their most general result---relies on a model whose depth \emph{scales} (logarithmically) with the string length. 
As such, no particular finite-size transformer construction can simulate FSAs on strings of \emph{arbitrary} length (which is required for equivalence). 
This is in contrast to our results, which feature transformers of \emph{fixed size} with respect to the string length. 
While \citet[Theorem 2]{liu2023transformers} also provides a result using a finite-depth transformer for a subset of FSAs, that construction results in a network much bigger and deeper than ours.
For example, their construction would result in $\left(\nsymbols^{\ngr - 1}\right)^2 \left(\ngr - 1\right) \log{\nsymbols}$ layers in contrast to $\ngr - 1$ layers with a single head in our case. 
Our focus on \ngram LMs thus allows for a more compact and simpler representation. 
Worth noting is that despite being close to our analysis, none of the related work treats probability distributions over strings but rather focuses on the binary decision of whether a string belongs to a language or not based on the conditional probabilities output by the model, or, in the case of \citet{liu2023transformers}, only the computation of state sequences.
Transformer LMs are studied probabilistically by \citet{xie2022an}, who provide a Bayesian interpretation of their ability to learn \emph{in context} by studying the learning abilities of hidden Markov models (HMMs).
While HMMs are equivalent to probabilistic finite-state automata, \citet{xie2022an} do not connect the in-context learning ability to concrete representational capacity results.
Rather, they seek to \emph{explain} the behavior of in-context learning as implicit Bayesian inference.

\section{Conclusion}
We study the representational capacity of transformer LMs with \ngram LMs.
We show how the parallelizable nature of \ngram LMs is easy to capture with the transformer architecture and provide multiple lower bounds on the probabilistic representational capacity of transformer LMs.
Concretely, we show that transformer LMs can represent \ngram LMs both with hard and sparse attention, exhibiting multiple mechanisms transformer LMs can employ to simulate \ngram LMs.
Altogether, our results reinforce the utility of non-sequential models of computation for the study of transformers, particularly in the language modeling setting.\looseness=-1

\section*{Limitations}

We connect transformer LMs to \ngram LMs because of their parallelizable nature and their traditional popularity in NLP.
However, \ngram LMs describe a very simple class of LMs, meaning that the lower bounds are somewhat less relevant than the characterization in terms of more expressive formal models of computation would be.
Accordingly, we expect that the lower bounds are somewhat loose and that transformer LMs can represent more than \ngram LMs, which is also in line with the empirical success of transformer LMs.
We leave it to future work to tighten the established lower bounds.

As with most theoretical investigations of transformers, our results are strongest and the most precise in the hard attention setting.
However, hard attention is not used in practice, which limits the applicability of the results.
The constructions presented in this paper are also purely meant to showcase the existence of a mechanism that can be used to simulate \ngram LMs.
They do not suggest that the same mechanisms will be learned by models used in practice.
Indeed, the very sparse representations are not in line with the common dense contextual representations usually learned by trained models.\looseness=-1

We also only focus on \emph{lower bounds} of the representational capacity.
We do not consider any upper bounds and existing results for similar models to ours suggest that the lower bound is indeed somewhat loose \citep{yao-etal-2021-self}.\footnote{For example, we cannot say that the lower bounds imply any limitations of hard attention transformer LMs.}
That is, we expect that transformer LMs can represent much more than \ngram LMs, and expect that many of the existing results on the computational power of such models can be extended to the probabilistic setting.\looseness=-1

While we present a comprehensive analysis of transformer LMs in the context of \ngram LMs, we do not consider various aspects of the relationship that could be interesting.
This is done to keep the presentation focused and concise.
For example, we do not consider whether such simulations can be \emph{learned} from data, an interesting avenue for future research.
Lastly, note that while we focus here specifically on the commonly deployed transformer-based \emph{language models}, there are many other interesting applications of transformers, such as encoder-only acceptors of unweighted languages.
These applications are better covered by existing work.\looseness=-1

\section*{Ethics Statement}
The paper provides a way to theoretically analyze language models.
To the best knowledge of the authors, there are no ethical implications of this paper.\looseness=-1

\section*{Acknowledgements}
Ryan Cotterell acknowledges support from the Swiss National Science Foundation (SNSF) as part of the ``The Forgotten Role of Inductive Bias in Interpretability'' project.
Anej Svete is supported by the ETH AI Center Doctoral Fellowship.
We thank Tim Vieira and Shannon Veitch for helpful discussions and comments on the manuscript.

\bibliography{anthology,custom}

\newpage{}

\appendix
\onecolumn

\section{Modelling Assumptions}
As encouraged by \citet{strobl2023transformers}, we provide in \cref{tab:summary} a short summary of the assumptions behind the specific transformer architecture we consider in this work.
This aims to facilitate the placement of the results in the context of existing work and to make the results more accessible to the reader.
\begin{table}[h!] \centering\footnotesize
    \begin{tabular}{@{}lllllp{2.5cm}l@{}}
        \toprule
        Lower bound                          & PE          & Precision  & Attention                                                                                                & Attention    & Architecture                                                    & Notes                                                     \\
        \midrule
        $\ni$ \ngram LMs                       & $\begin{pmatrix}
                  \sqrt{\frac{1}{\tstep + k}} \\
                  \sqrt{1 - \frac{1}{\tstep + k}}
              \end{pmatrix}_{k = 0, \ldots, \ngr - 1}$ & $\Q$, $\R$ & hard & decoder-only & $\ngr - 1$ heads, $1$ layer                                      & \cref{thm:transformers-n-gram-label}                    \\
        $\ni$ \ngram LMs                      & $\begin{pmatrix}
                  \sqrt{\frac{1}{\tstep}} \\
                  \sqrt{1 - \frac{1}{\tstep}} \\
                  \sqrt{\frac{1}{\tstep + 1}} \\
                  \sqrt{1 - \frac{1}{\tstep + 1}}
              \end{pmatrix}$ & $\Q$, $\R$ & hard                                                                                                     & decoder-only & $1$ head, $\ngr - 1$ layers                                       & \cref{thm:transformers-n-gram-label-multi-layer}        \\
        $\ni$  \ngram LMs                     & $1, \tstep, 10^{-\tstep}$         & $\Q$, $\R$ & hard                                                                                                     & decoder-only & $1$ head, $1$ layer                                             & \cref{thm:transformers-n-gram-single} \\
        $\ni$  \ngram LMs                     & $1, \tstep$         & $\Q$, $\R$ & sparse                                                                                                     & decoder-only & $\ngr - 1$ heads, $1$ layer                                             & \cref{thm:transformers-n-gram-label-sparse} \\
        \bottomrule
    \end{tabular}
    \caption{A summary of the main assumptions about the models in the style of \citet[][Table 1]{strobl2023transformers}. Hard attention here refers to average-hard in the vocabulary of \citet{strobl2023transformers}.}
    \label{tab:summary}
\end{table}

\section{Proofs: Hard Attention}
This section provides detailed proofs of all theorems about the representational capacity of hard-attention transformer LMs stated in the main part of the paper.
\subsection{Computing Logical \texorpdfstring{\texttt{AND}}{AND} with an MLP} \label{sec:and}

\begin{definition} \label{def:mlp}
A $\ReLU$-activated \defn{multi-layer-perceptron} (MLP) $\mlp\colon\R^N\to\R^M$ is a function defined as the composition of functions $\vfunc_1, \cdots, \vfunc_L$ 
\begin{equation}
    \mlp\left(\vx\right) \defeq \left( \vfunc_L \circ \vfunc_{L-1} \circ \cdots \circ \vfunc_1\right) \left(\vx\right)
\end{equation}
where each $\vfunc_\ell$ for $\ell \in \NTo{L}$ is defined as
\begin{subequations}
\begin{align}
    \vfunc_\ell\left(\vx\right) &\defeq \ReLU\left(\mW_\ell\vx + \vb_\ell \right)\quad \ell\in \NTo{L-1}\\
    \vfunc_L\left(\vx\right) &\defeq \mW_L\vx + \vb_L
\end{align}
\end{subequations}
where $\mW_\ell\in\R^{N_\ell\times M_\ell}$ is a weight matrix with dimensions $N_\ell$ and $M_\ell$ specific to layer $\ell$, $\vb_\ell\in\R^{M_\ell}$ is a bias vector.
We refer to MLPs by the number of hidden layers, e.g., a one-layer-MLP is an MLP with one \emph{hidden} layer.
\end{definition}

In our construction, simulating a \ngram LM with a transformer LM requires a component of the transformer to perform the logical \texttt{AND} operation between specific entries of binary vectors $\vx \in \B^\hiddDim$.
The following lemma shows how this can be performed by an MLP with appropriately set parameters.
\begin{lemma} \label{fact:and}
    Consider $\idxm$ indices $\idxi_1, \ldots, \idxi_\idxm\in \NTo{\hiddDim}$ and vectors $\vx, \vv \in \B^\hiddDim$ such that
    \begin{equation}
        \evv_{\idxi} = \ind{\idxi \in \set{\idxi_1, \ldots, \idxi_\idxm}},
    \end{equation}
    i.e., with entries $1$ at indices $\idxi_1, \ldots, \idxi_\idxm$.
    Then, it holds for the MLP $\mlp\left(\vx\right) \defeq \ReLU\left(\vv^\top \vx - \left(\idxm - 1\right) \right)$ that 
    \begin{equation}
        \mlp\left(\vx\right) = 1 \text{ if and only if } \evx_{\idxi_\idxk} = 1 \text{ for all } \idxk = 1, \ldots, \idxm.
    \end{equation}
    In other words,
    \begin{equation}
        \mlp\left(\vx\right) = \evx_{\idxi_1} \wedge \cdots \wedge \evx_{\idxi_\idxm}.
    \end{equation}
\end{lemma}
\begin{proof}
    By the definition of $\vv$, $\inner{\vv}{\vx} \leq \idxm$ for all $\vx \in \B^\hiddDim$.
    Furthermore, $\inner{\vv}{\vx} = \idxm$ if and only if $\evx_{\idxi_\idxk} = 1$ for all $\idxk = 1, \ldots, \idxm$.
    The $\ReLU$ function maps all other values to $0$ and does not change the output value $1$ in this case.
\end{proof}

\subsection{Proofs of the Hard Attention Case} \label{sec:proofs-hard-attention}
This subsection contains all proofs of the representational capacity of hard attention transformer LMs.
We tackle three cases: the simulation with $\ngr - 1$ heads and a single layer, with $\ngr - 1$ layers and a single head, and lastly, the simulation with a single head and a single layer.
All sections first outline the intuition behind the proofs and then provide the details in the form of finer-grained lemmata.

\subsubsection{Simulation with \texorpdfstring{$\ngr - 1$}{n-1} Heads: The Intuition} \label{sec:construction-intuition}
We now outline the intuition behind the construction of a hard attention transformer LM simulating an \ngram LM, as first presented in \cref{fig:transformer-n-gram-label}.\footnote{For simplicity, we disregard the role of residual connections in the following outline. Residual connections are, however, considered in the full proof later.}
To ease the exposition, we start with the final step of the construction: Assuming we have identified the appropriate history $\str^{\tstep - 1}_{\tstep - \ngr + 1}$ after combining the head values using the head combining function $\tfheadCombine$, we show how $\tfpLM$ can encode the conditional probability distribution $\pLNSM\left(\sym_\tstep\mid\str_{\tstep- \ngr + 1 : \tstep - 1}\right)$.
The intuition of this step is simple: Knowing what the individual $\pLNSM\left(\sym_\tstep\mid\str^{\tstep - 1}_{\tstep - \ngr + 1}\right)$ for $\sym_\tstep \in \eosalphabet$ are, we can simply put their logits into a vector and combine the constructed vectors for all possible histories into the output matrix $\outMtx$:\footnote{To be able to take the $\log$ of $0$ probabilities, we work over the set of \emph{extended} reals $\Rex = \R \cup \set{-\infty, \infty}$.}\textsuperscript{,}\footnote{Throughout the paper, we implicitly index the matrices directly with symbols and histories. We assume that the symbols and histories are ordered in some way and that the matrices are ordered accordingly.}
\begin{equation}
    \eOutMtx_{\sym, \str^{\tstep - 1}_{\tstep - \ngr + 1}} \defeq \log{\pLNSM\left(\sym_\tstep\mid\str^{\tstep - 1}_{\tstep - \ngr + 1}\right)}
\end{equation}
In the following, we write $\tfencfun\left(\str_{<\tstep}\right)$ as shorthand notation for $\tfencfun\left(\str_{<\tstep}\right) \defeq \fTransf\left(\vx_{\tstep - 1}^\tfnumlayer\right)$ (i.e., the representation which is linearly transformed by $\outMtx$ to compute $\pLM\left(\symt \mid \strlt\right)$ after normalization) where $\mX^\tfnumlayer = \tf\left(\staticRepr\right)\left(\strlt\right)$.
If we one-hot encode the identified history with $\tf$ as
\begin{equation}
    \tfencfun\left(\str_{<\tstep}\right) \defeq \onehot{\str^{\tstep - 1}_{\tstep - \ngr + 1}}
\end{equation}
we can then, using the formulation of the transformer LM from \cref{def:transformer-plnsm}, use the $\tfencfun\left(\str_{<\tstep}\right)$ to look up the appropriate column in $\outMtx$ containing the logits of the conditional probabilities given the identified history for all possible $\sym_\tstep \in \eosalphabet$, i.e., $\left(\outMtx \ \tfencfun\left(\str_{<\tstep}\right)\right)_\sym = \log\pLNSM\left(\sym \mid\str^{\tstep - 1}_{\tstep - \ngr + 1}\right)$.

We now consider the preceding step of the simulation: Identifying the history given that the $\ngr - 1$ heads identified the symbols $\sym_1, \ldots, \sym_{\ngr - 1}$ in the positions they attended to.
If we concatenate the values of the $\ngr - 1$ heads into a vector $\vv$, this vector of size $\left(\ngr - 1\right)|\eosalphabet|$ will contain the \defn{multi-hot} representation of the history of interest:
\begin{equation}
    \vv = \begin{pmatrix}
        \onehot{\sym_1} \\
        \vdots          \\
        \onehot{\sym_{\ngr - 1}}
    \end{pmatrix}
\end{equation}
and $\vv_{\idxi |\eosalphabet| + \idxj} = 1$ if and only if $\symordering\left(\sym_\idxi\right) = \idxj$ for a bijection $\symordering\colon \eosalphabet \to \NTo{|\eosalphabet|}$ that determines the indices of the one-hot representations of the symbols.
We would then like to transform this vector into a vector $\vu \in \R^{|\alphabet|^{\ngr - 1}}$ such that
\begin{equation}
    \evu_{\idxi} = 1 \iff \idxi = \anOrdering\left(\sym_1, \ldots, \sym_{\ngr - 1}\right)
\end{equation}
for a bijection $\anOrdering\colon \underbrace{\eosalphabet \times \cdots \times \eosalphabet}_{\ngr - 1\text{ times}} \to \NTo{|\eosalphabet|^{\ngr - 1}}$.
This can be equivalently written as
\begin{equation}
    \evu_{\idxi} = 1 \iff \evv_{\idxj |\eosalphabet| + \symordering\left(\sym_\idxj\right)} = 1 \text{ for all } \idxj = 1, \ldots, \ngr - 1
\end{equation}
where $\idxi = \anOrdering\left(\sym_1, \ldots, \sym_{\ngr - 1}\right)$.
This is an instance of performing the logical \texttt{AND} operation, which can be implemented by an MLP as described in \cref{fact:and}.
This MLP will form the transformation $\tfheadCombine$ combining the information obtained from all the heads of the transformer.

This brings us to the final part of the proof: Identifying the symbols at the previous $\ngr - 1$ positions by the $\ngr - 1$ transformer heads.
To show how this can be done, let us consider the parameters we can still set to define a transformer:
\begin{itemize}[itemsep=0.5pt,parsep=0.5pt]
    \item The position-augmented symbol representations $\posInEmbedding$.
          Inspired by concurrent work from \citet{merrill2023expressive}, we use the representations of the form
          \begin{equation}
              \posInEmbeddingFun{\symt, \tstep} = \begin{pmatrix}
                  \onehot{\symt} \\
                  \sqrt{\frac{1}{\tstep}} \\
                  \sqrt{1 - \frac{1}{\tstep}} \\
                  \sqrt{\frac{1}{\tstep + 1}} \\
                  \sqrt{1 - \frac{1}{\tstep + 1}} \\
                  \vdots \\
                  \sqrt{\frac{1}{\tstep + \ngr - 1}} \\
                  \sqrt{1 - \frac{1}{\tstep + \ngr - 1}}
              \end{pmatrix} \in \R^{2 \ngr}
          \end{equation}
          This results in vectors $\posInEmbedding\left(\sym_\tstep, \tstep\right)$ of size $|\eosalphabet| + 2 + 2 \left(\ngr - 1\right)$.
          Note that such a symbol representation function can be implemented by concatenating or adding symbol- ($\onehot{\symt}$) and position- ($\sqrt{\frac{1}{\tstep}}, \ldots, \sqrt{1 - \frac{1}{\tstep + \ngr - 1}}$) specific components, which is in line with most practical implementations of the transformer architecture.
    \item The attention scoring function $\tfscorefun$.
          We will use the standard dot-product scoring function
          \begin{equation} \label{eq:transformer-ngram-function}
              \tfscorefun\left(\vq, \vk\right) \defeq \innerProd{\vq}{\vk}.
          \end{equation}
          $\tfscorefun$ will, together with the positional encodings, allow us to easily single out the relevant positions in the string.
    \item The parameters of each of the attention heads, that is, the transformations $\qTransf$, $\kTransf$, and $\vTransf$.
          Each of those will take the form of a linear transformation of the symbol (and positional) representations.
          We describe them and their roles in more detail below.
\end{itemize}
The parameters of all the heads will be identical, with the only difference being a single parameter that depends on the ``index'' of the head, $h$.
In the following, we describe the construction of a single head.
At any time step $\tstep$ (i.e., when modeling the conditional distribution $\pLNSM\left(\sym_\tstep \mid \str_{<\tstep}\right)$), the head $h$ will attend to the symbol at position $\tstep - h$, $\sym_{\tstep - h}$.
In \cref{fig:transformer-n-gram-label}, for example, \texttt{Head $3$} attends to the position $\tstep - 3$, which is denoted by the stronger arrow to that position.
We now describe the individual transformations $\qTransf_h$, $\kTransf_h$, $\vTransf_h$, and $\oTransf_h$ of the head $h$.
All of them will be \emph{affine} transformations.
Since we are considering only the first layer of the transformer, we can think of the inputs to the layer as the original symbol representations together with their position encodings (rather than some contextual representations at higher levels).
As mentioned, the head $h$ will be responsible for identifying the symbol at position $\tstep - h$.
Therefore, we want it to put all its attention to this position.
In other words, given the query $\vq_{\tstep - 1}$, we want the attention function in \cref{eq:transformer-ngram-function} to be uniquely maximized by the key of the symbol at position $\tstep - h$.
Notice that, therefore, the key does not have to depend on the identity of the symbol at position $\tstep - h$---only the positional information matters.
Let us then consider the following query and key transformations for head $h$:\looseness=-1
\begin{align}
    \qTransf_h\colon & \posInEmbedding\left(\sym_\tstep, \tstep\right) \mapsto \begin{pmatrix}
      \sqrt{\frac{1}{\tstep}} \\
      \sqrt{1 - \frac{1}{\tstep}}
    \end{pmatrix} \\
    \kTransf_h\colon & \posInEmbedding\left(\sym_\tstep, \tstep\right) \mapsto \begin{pmatrix}
      \sqrt{\frac{1}{\tstep + h}} \\
      \sqrt{1 - \frac{1}{\tstep + h}}
    \end{pmatrix}.
\end{align}
Given such a query and such keys, the scoring function computes
\begin{equation} \label{eq:hard-attention-scores}
    \tfscorefun\left(\vq_\tstep, \vk_\idxj\right) = \innerProd{\begin{pmatrix}
      \sqrt{\frac{1}{\tstep}} \\
      \sqrt{1 - \frac{1}{\tstep}}
    \end{pmatrix}}{\begin{pmatrix}
      \sqrt{\frac{1}{\idxj + h}} \\
      \sqrt{1 - \frac{1}{\idxj + h}}
    \end{pmatrix}}.
\end{equation}
\cref{eq:hard-attention-scores} is an inner product between two unit vectors, and is therefore maximized if and only if they are the same, that is, if $\idxj = \tstep - h$. This is exactly the position that we want the head $h$ to attend to.\footnote{Note that while the choice of the positional encodings in this construction is uncommon in practice, the popular sinusoidal positional encodings \citep{Vaswani2017} have also been linked to the ability of the transformer-based models to attend to specific positions of interest based on linear transformations of the positional encodings \citep{Vaswani2017}.}
Intuitively, both transformations keep only the positional information.
The query transformation ``injects'' the knowledge of which position should maximize the attention score, while the key transformation simply ``exposes'' the positional information about the symbol.
The constants $1$ and $-1$ and the index of the position ensure that the inner product simply computes the difference between the position of the symbol and the position of interest.

This leaves us with the question of how to use the position of the symbol of interest ($\tstep - h$) to extract the one-hot encoding of the symbol at that position.
Due to the information contained in the symbol representations $\posInEmbedding\left(\sym_\idxj\right)$, this is trivial:
\begin{equation} \label{eq:hard-attention-values}
    \vTransf\colon \posInEmbedding\left(\sym_\tstep, \tstep\right) \mapsto \onehot{\sym_\idxj}.
\end{equation}
With this, the identity of the symbol is carried forward through the attention mechanism.
Notice that the only head-depend transformation is the query transformation---it depends on the index of the head, determining the position of interest, meaning that every head defines a different query transformation, while the keys and values transformations are the same among all heads.
This concludes the outline of the proof.

\subsubsection{Simulation with \texorpdfstring{$\ngr - 1$}{n-1} Heads: Proofs}
This subsection formally proves the construction intuited in \cref{sec:construction-intuition} by proving a sequence of lemmata that formalize each of the steps described in the intuition.
Specifically,
\begin{enumerate}
    \item \cref{lemma:hard-attention-lemma-2} shows how the one-hot encodings of individual symbols in the history can be combined into the one-hot encoding of the history.
    \item \cref{lemma:hard-attention-lemma-3-1} shows that the scoring function is maximized at the position of interest.
    \item \cref{lemma:hard-attention-lemma-3} shows how the one-hot encodings of the symbols in the history can be identified by the hard attention mechanism.
    \item The proof of \cref{thm:transformers-n-gram-label} shows how the construction of the one-hot encoding of the current history allows us to define the appropriate next-symbol conditional distribution of the \ngram LMs.
\end{enumerate}

\begin{lemma} \label{lemma:hard-attention-lemma-2}
    Let $\tf$ be a transformer with $\tfheadnum = \ngr - 1$ heads.
    Let $\vz_h = \onehot{\sym_{\tstep - h}}$ be the output of the $h\textsuperscript{th}$ head at time $\tstep$.
    Then, there exists a function $\tfheadCombine$ implemented by a single-layer MLP such that
    \begin{equation} \label{eq:hard-attention-lemma-2}
        \tfheadCombine\left(\vz_1, \ldots, \vz_{\ngr - 1}\right) = \onehot{\str^{\tstep - 1}_{\tstep - \ngr + 1}}.
    \end{equation}
\end{lemma}
\begin{proof}
    Let $\str = \sym_1 \ldots \sym_{\ngr - 1} \in \bosalphabet^{\ngr - 1}$ and let $\idxi$ be the index in $\NTo{|\bosalphabet|^{\ngr - 1}}$ that corresponds to $\str$.
    Furthermore, let $\idxi_1, \ldots, \idxi_{\ngr - 1}$ be the indices corresponding to the symbols $\sym_1, \ldots, \sym_{\ngr - 1}$ in $\vz_1, \ldots, \vz_{\ngr - 1}$. 
    Then, we have
    \begin{equation}
        \onehot{\str^{\tstep - 1}_{\tstep - \ngr + 1}}_{\idxi} = 1 \iff \evz_{1, \idxi_1} = 1 \wedge \cdots \wedge \evz_{\ngr - 1, \idxi_{\ngr - 1}} = 1 
    \end{equation}
    \cref{eq:hard-attention-lemma-2} is an instance of the logical \texttt{AND} operation on the indices encoding the individual histories, which can be implemented by an MLP as shown in \cref{fact:and}.
\end{proof}

The next lemma presents a useful equality about the standard dot-product attention scoring function: For unit vectors, the attention score is maximized if and only if the vectors are identical. 
We will use this fact in our construction to attend to particular positions in the string.
\begin{lemma} \label{lemma:hard-attention-lemma-3-1}
    Given a fixed $\tstep \in \N$, $\tfscorefun$, define
    \begin{equation} 
        \funcg\left(\idxj\right) \defeq \innerProd{\begin{pmatrix}
      \sqrt{\frac{1}{\tstep - 1}} \\
      \sqrt{1 - \frac{1}{\tstep - 1}}
    \end{pmatrix}}{\begin{pmatrix}
      \sqrt{\frac{1}{\idxj + h}} \\
      \sqrt{1 - \frac{1}{\idxj + h}}
    \end{pmatrix}}
    \end{equation} 
    for $\idxj \in \NTo{\tstep - 1}$. 
    Then, $\funcg$ is maximized at $\idxj = \tstep - 1 - h$:
    \begin{equation} \label{eq:inner-prod-lemma}
        \argmax_{\idxj \in \NTo{\tstep - 1}}\left(\innerProd{\begin{pmatrix}
      \sqrt{\frac{1}{\tstep - 1}} \\
      \sqrt{1 - \frac{1}{\tstep - 1}}
    \end{pmatrix}}{\begin{pmatrix}
      \sqrt{\frac{1}{\idxj + h}} \\
      \sqrt{1 - \frac{1}{\idxj + h}}
    \end{pmatrix}}\right) = \tstep - 1 - h.
    \end{equation}
\end{lemma}
\begin{proof}
    The two arguments to the inner product in \cref{eq:inner-prod-lemma} are unit vectors. 
    Inner products of unit vectors are at most $1$, with the maximum achieved only if the two vectors are identical.
    This means that the function in \cref{eq:inner-prod-lemma} is maximized if and only if 
    \begin{subequations}
    \begin{align}
        \begin{pmatrix}
            \sqrt{\frac{1}{\tstep - 1}} \\
            \sqrt{1 - \frac{1}{\tstep - 1}}
        \end{pmatrix} &= \begin{pmatrix}
            \sqrt{\frac{1}{\idxj + h}} \\
            \sqrt{1 - \frac{1}{\idxj + h}}
        \end{pmatrix} \iff \label{eq:matrixEquality} \\
        \sqrt{\frac{1}{\tstep - 1}} &= \sqrt{\frac{1}{\idxj + h}} \iff \label{eq:sqrtEquality} \\
        \frac{1}{\tstep - 1} &= \frac{1}{\idxj + h} \iff \label{eq:fracEquality} \\
        \idxj &= \tstep - 1 - h. \label{eq:finalEquality}
    \end{align}
    \end{subequations}
\end{proof}

The following lemma presents the core of the proof of \cref{thm:transformers-n-gram-label}, exhibiting the construction of a transformer head that can single out and one-hot encode a symbol at a specific position in the input string.
The lemma relies on a simple pre-processing of the input string, where the string is prepended (padded) with $\ngr - 1$ \underline{b}eginning \underline{o}f \underline{s}tring symbols $\sym_1 = \ldots = \sym_{\ngr - 1} \defeq \bos$, which is common practice in language modeling literature, especially when talking about \ngram LMs.
We will denote $\bosalphabet \defeq \alphabet \cup \set{\bos}$.
This enables a cleaner presentation of the concrete construction of the attention mechanism.

\paragraph{The idea of the proof.}
\cref{lemma:hard-attention-lemma-3} contains a number of technical definitions of the parameters of a transformer layer (cf. \cref{def:transformer-layer}.
Together, they describe a single head of a transformer layer (which will contain $\tfheadnum = \ngr - 1$ such heads) that is able to extract the one-hot encoding of a particular symbol in the history.
The layer of $\ngr - 1$ heads will then be able to extract the $\ngr - 1$ symbols, as required by \cref{lemma:hard-attention-lemma-2}.
We now describe a single head more formally.
Define the following position-augmented symbol representation function of the transformer head $h$:
  \begin{equation} \label{eq:single-head-eq-1}
      \posInEmbeddingFun{\sym, \tstep} = \begin{pmatrix}
          \onehot{\sym} \\ 
          \zero_{\bosnsymbols} \\
          \sqrt{\frac{1}{\tstep}} \\
          \sqrt{1 - \frac{1}{\tstep}} \\
          \sqrt{\frac{1}{\tstep + 1}} \\
          \sqrt{1 - \frac{1}{\tstep + 1}} \\
          \vdots \\
          \sqrt{\frac{1}{\tstep + \ngr - 1}} \\
          \sqrt{1 - \frac{1}{\tstep + \ngr - 1}}
      \end{pmatrix} \in \R^{2 \bosnsymbols + 2 \ngr}.
  \end{equation}
Here, $\onehot{\cdot} \in \set{0, 1}^\bosnsymbols$ one-hot encodes symbols over $\bosalphabet$.
This means that the entire static representations contain multiple components:
\begin{itemize}
    \item two $\bosnsymbols$-dimensional slots for \emph{symbol} representations and
    \item $\ngr$ 2-dimensional slots for head-specific \emph{positional} representations.
\end{itemize}
We further define
\begin{equation}
    \tfscorefun\left(\vq, \vk\right) \defeq \innerProd{\vq}{\vk},
\end{equation}
\begin{subequations}
    \begin{alignat}{2}
        \qTransf\left(\vx\right) & \defeq \mQ \vx, \qquad &&\mQ \in \R^{2 \times \left(2 \bosnsymbols + 2\ngr\right)},            \\
        \kTransf\left(\vx\right) & \defeq \mK \vx, \qquad &&\mK \in \R^{2 \times \left(2 \bosnsymbols + 2\ngr\right)},                                \\
        \vTransf\left(\vx\right) & \defeq \mV \vx, \qquad &&\mV \in \R^{\left(2 \bosnsymbols + 2\ngr\right) \times \left(2 \bosnsymbols + 2\ngr\right)},                              \\
        \oTransf\left(\vx\right) & \defeq \mO \vx, \qquad &&\mO \in \R^{\left(2 \bosnsymbols + 2\ngr\right) \times \left(2 \bosnsymbols + 2\ngr\right)},
    \end{alignat}
\end{subequations}
\begin{subequations}
    \begin{align}
        \mQ_{:, 2 \bosnsymbols + 1: 2 \bosnsymbols + 2}             & \defeq \mI_2 \\
        \mK_{:, 2 \bosnsymbols + 2 h + 1: 2 \bosnsymbols + 2 h + 2}             & \defeq \mI_2,                   \\
        \mV_{\bosnsymbols + 1: 2 \bosnsymbols, 1: \bosnsymbols} & \defeq \mI_\bosnsymbols \\
        \mO_{1: \bosnsymbols, 1:\bosnsymbols} &\defeq -\mI_{|\bosalphabet|} \label{eq:single-head-last-eq}
    \end{align}
\end{subequations}
We can visualize these matrices as 
\begin{subequations}
    \begin{align}
        \mQ &= \begin{blockarray}{ccccc}
            {\scriptstyle \bosnsymbols} & {\scriptstyle \bosnsymbols} & {\scriptstyle 2} & {\scriptstyle \left(\ngr - 1\right)2} & \\
            \overbrace{\hspace{1.5cm}}^{\text{$1\textsuperscript{st}$ symbol slot}} & \overbrace{\hspace{1.5cm}}^{\text{$2\textsuperscript{nd}$ symbol slot}} & \overbrace{\hspace{1.5cm}}^{\text{$0\textsuperscript{th}$ position slot}} & \overbrace{\hspace{1.5cm}}^{\text{$\ngr - 1$ position slots}} & \\
            \begin{block}{(cccc)c}
                & & \mI_2 & & {\scriptstyle 2} \\
            \end{block}
        \end{blockarray} \\
        \mK &= \begin{blockarray}{cccccc}
            {\scriptstyle \bosnsymbols} & {\scriptstyle \bosnsymbols} & {\scriptstyle 2 h} & {\scriptstyle 2} & {\scriptstyle \left(\ngr - 1 - h\right) 2} & \\
            \overbrace{\hspace{1.5cm}}^{\text{$1\textsuperscript{st}$ symbol slot}} & \overbrace{\hspace{1.5cm}}^{\text{$2\textsuperscript{nd}$ symbol slot}} & \overbrace{\hspace{1.5cm}}^{\text{$h + 1$ position slots}} & \overbrace{\hspace{1.5cm}}^{\text{$h\textsuperscript{th}$ position slot}} & \overbrace{\hspace{1.5cm}}^{\text{$\ngr - h$ position slots}} & \\
            \begin{block}{(ccccc)c}
                & & & \mI_2 & & {\scriptstyle 2} \\
            \end{block}
        \end{blockarray} \\
        \mV &= \begin{blockarray}{cccc}
            \overbrace{\hspace{1.5cm}}^{\text{$1\textsuperscript{st}$ symbol slot}} & \overbrace{\hspace{1.5cm}}^{\text{$2\textsuperscript{nd}$ symbol slot}} & \overbrace{\hspace{1.5cm}}^{\text{$\ngr$ position slots}} & \\
            \begin{block}{(ccc)c}
                & & & \text{\small $1\textsuperscript{st}$ symbol slot ($\bosnsymbols$)} \\
                 \mI_\bosnsymbols & & & \text{\small $2\textsuperscript{nd}$ symbol slot ($\bosnsymbols$)} \\
                & & & \text{\small $\ngr$ position slots ($2\ngr$)} \\
            \end{block}
        \end{blockarray} \\
        \mO &= \begin{blockarray}{cccc}
            \overbrace{\hspace{1.5cm}}^{\text{$1\textsuperscript{st}$ symbol slot}} & \overbrace{\hspace{1.5cm}}^{\text{$2\textsuperscript{nd}$ symbol slot}} & \overbrace{\hspace{1.5cm}}^{\text{$\ngr$ position slots}} & \\
            \begin{block}{(ccc)c}
                -\mI_\bosnsymbols & & & \text{\small $1\textsuperscript{st}$ symbol slot ($\bosnsymbols$)} \\
                & & & \text{\small $2\textsuperscript{nd}$ symbol slot ($\bosnsymbols$)} \\
                & & & \text{\small $\ngr$ position slots ($2 \ngr$)} \\
            \end{block}
        \end{blockarray}
    \end{align}
\end{subequations}
where $\zero_N$ is a $N$-dimensional vector of zeros, $\mI_N$ is the $N$-dimensional identity matrix and the unspecified elements of $\mQ, \mK$, $\mV$, $\mO$ are $0$.

\begin{lemma} \label{lemma:hard-attention-lemma-3}
    Let $\alphabet$ be an alphabet and $\str \in \kleene{\alphabet}$.
    For any $\tstep = 1, \ldots, |\str|$, the $h\textsuperscript{th}$ transformer head ($h \in \NTo{\ngr - 1}$) defined with the parameters specified in \cref{eq:single-head-eq-1} to \cref{eq:single-head-last-eq} outputs
    \begin{equation} \label{eq:head-output}
        \vz_{\tstep - 1} = \begin{pmatrix}
            \zero_{\bosnsymbols}                                                     \\
            \onehot{\sym_{\tstep - 1 - h}} \\
            \zero_{2 \ngr}
        \end{pmatrix}.
    \end{equation}
    In particular, this means that the output $\vz_{\tstep - 1}$ at time step $\tstep - 1$ contains the one-hot encoding of the symbol at position $\tstep - h - 1$.%
    \footnote{For technical reasons---the residual connections---the output is not $\onehot{\sym_{\tstep - 1 - h}}$ but larger (with additional zeros), as shown in \cref{eq:head-output}. This, however, is equivalent for the purposes of \cref{lemma:hard-attention-lemma-2} and later \cref{thm:transformers-n-gram-label}.}
\end{lemma}
\begin{proof}
    Fix $\tstep \in \NTo{|\str|}$.
    We compute the representation of $\strlt$ computed by the head.
    First, observe that the matrix defining the query function $\qTransf$ projects onto the first two components of the positional encoding from the representation of $\symtminus$:
    \begin{equation}
        \qTransf\left(\posInEmbedding\left(\sym_{\tstep - 1}\right)\right) = \mQ \; \posInEmbedding\left(\sym_{\tstep - 1}\right) = \begin{pmatrix}
              \sqrt{\frac{1}{\tstep - 1}} \\
              \sqrt{1 - \frac{1}{\tstep - 1}} \\
        \end{pmatrix}.
        \end{equation}
    Similarly, the matrix defining the key transformation projects onto the $h\textsuperscript{th}$ positional encoding slot, i.e., the dimensions $ 2 \bosnsymbols + 2 h + 1$ and $2 \bosnsymbols + 2 h + 2$:
        \begin{equation}
        \kTransf\left(\posInEmbedding\left(\sym_\idxj\right)\right) = \mK \; \posInEmbedding\left(\sym_\idxj\right)  = \begin{pmatrix}
              \sqrt{\frac{1}{\idxj + h}} \\
              \sqrt{1 - \frac{1}{\idxj + h}} \\
       \end{pmatrix}
       \end{equation}
    for $\idxj = -\ngr + 2, \ldots, \tstep - 1$.
       
    This results in the scoring function
    \begin{equation} \label{eq:hard-attention-lemma-3-1}
        \tfscorefun\left(\vq_{\tstep - 1}, \vk_\idxj\right) \defeq \innerProd{\vq_{\tstep - 1}}{\vk_\idxj} = \innerProd{\begin{pmatrix}
      \sqrt{\frac{1}{\tstep - 1}} \\
      \sqrt{1 - \frac{1}{\tstep - 1}}
    \end{pmatrix}}{\begin{pmatrix}
      \sqrt{\frac{1}{\idxj + h}} \\
      \sqrt{1 - \frac{1}{\idxj + h}}
    \end{pmatrix}}.
    \end{equation}
    In particular, as shown in \cref{lemma:hard-attention-lemma-3-1}, $\tfscorefun$ is maximized for 
    \begin{equation}
        \idxj = \tstep - 1 - h.
    \end{equation}
    This means that
    \begin{equation} \label{eq:hardmax-max}
        \hardmax\left(\tfscorefun\left(\vq_{\tstep - 1}, \vk_1\right), \tfscorefun\left(\vq_{\tstep - 1}, \vk_2\right), \dots ,\tfscorefun\left(\vq_{\tstep - 1}, \vk_{\tstep - 1}\right) \right)_\idxj = \ind{\idxj = \tstep - 1 - h}.
    \end{equation}
    The definition of $\mV$ further means that 
    \begin{equation}       \vTransf\left(\posInEmbedding\left(\sym_\idxj\right)\right) = \mV \; \posInEmbedding\left(\sym_\idxj\right) = \begin{pmatrix}
          \zero_{\bosnsymbols} \\
          \onehot{\sym_\idxj}  \\
          \zero_{2 \ngr}
      \end{pmatrix},
    \end{equation}
    giving us, by \cref{eq:attn-block-1},
    \begin{equation}
        \va_{\tstep - 1} = \vv_{\tstep - 1 - h} + \vx_{\tstep - 1} = \begin{pmatrix}
            \onehot{\sym_{\tstep - 1}}                                                      \\
            \onehot{\sym_{\tstep - 1 - h}} \\
            \zero_{2 \ngr}
        \end{pmatrix}.
    \end{equation}    
    The definition of $\oTransf$ then gives us
\begin{subequations}
    \begin{align} \label{eq:hard-attention-z}
        \vz_{\tstep - 1} 
        &= \oTransfFun{\va_{\tstep - 1}} + \va_{\tstep - 1} \\
        &= \mO \va_{\tstep - 1} + \va_{\tstep - 1} \\
        &= \begin{pmatrix}
            -\mI_\bosnsymbols & & \\
            & & & \\
            & & &
        \end{pmatrix} \begin{pmatrix}
            \onehot{\sym_{\tstep - 1}}                                                      \\
            \onehot{\sym_{\tstep - 1 - h}} \\
            \zero_{2 \ngr}
        \end{pmatrix} + \begin{pmatrix}
            \onehot{\sym_{\tstep - 1}}                                                      \\
            \onehot{\sym_{\tstep - 1 - h}} \\
            \zero_{2 \ngr}
        \end{pmatrix} \\
        &= \begin{pmatrix}
            -\onehot{\sym_{\tstep - 1}}                                                      \\
            \zero_{|\bosalphabet|} \\
            \zero_{2 \ngr}
        \end{pmatrix} + \begin{pmatrix}
            \onehot{\sym_{\tstep - 1}}                                                      \\
            \onehot{\sym_{\tstep - 1 - h}} \\
            \zero_{2 \ngr}
        \end{pmatrix} \\
        &= \begin{pmatrix}
            \zero_{\bosnsymbols}                                                     \\
            \onehot{\sym_{\tstep - 1 - h}} \\
            \zero_{2 \ngr}
        \end{pmatrix},
    \end{align}
\end{subequations}
    which is what we wanted to prove.
\end{proof}

\cref{lemma:hard-attention-lemma-2,lemma:hard-attention-lemma-3} show that we can define a transformer that one-hot encodes the history of interest $\str^{\tstep - 1}_{\tstep - \ngr + 1}$.
We now show how to define an output matrix $\outMtx$ to define a weakly equivalent transformer LM.
Concretely, we define $\outMtx \in \R^{|\eosalphabet| \times \bosnsymbols^{\ngr - 1}}$ with
\begin{equation} \label{eq:def-outmtx}
    \eOutMtx_{\sym, \str^{\tstep - 1}_{\tstep - \ngr + 1}} \defeq \log{\pLNSM\left(\sym_\tstep\mid\str^{\tstep - 1}_{\tstep - \ngr + 1}\right)}.
\end{equation}

\singleLayerKHeadsTheorem*
\begin{proof}
    Let $\tf$ be a transformer LM with $\ngr - 1$ heads  defined with the parameters specified in \cref{eq:single-head-eq-1} to \cref{eq:single-head-last-eq} ($h \in \NTo{\ngr - 1}$).
    Let $\str \in \set{\bos}^{\ngr - 1}\kleene{\alphabet}$ with $|\str| = \strlen$ be a string and $\mX^\tfnumlayer = \left(\vx_{-\ngr + 2}^{\tfnumlayer\top}; \ldots; \vx_\strlen^{\tfnumlayer\top}\right) = \tf\left(\str\right)$.
    We derive:
    \begin{subequations}
    \begin{align}
        \tfpLM\left(\str\right) & = \tfpLM\left(\eos\mid\str\right) \prod_{\tstep = 1}^{\strlen} \tfpLM\left(\symt \mid \strlt\right) \justification{Autoregressive LM.} \\
        & = \softmaxfunc{\outMtx \; \fTransf\left(\vx_{\strlen}^\tfnumlayer\right)}{\eos} \prod_{\tstep = 1}^{\strlen} \softmaxfunc{\outMtx \; \fTransf\left(\vx_{\tstep - 1}^\tfnumlayer\right)}{\symt} \justification{\cref{def:transformer-plnsm}.} \\
        & = \softmaxfunc{\outMtx \; \onehot{\str_{\strlen - \ngr + 2 : \strlen}}}{\eos} \prod_{\tstep = 1}^{\strlen} \softmaxfunc{\outMtx \; \onehot{\str^{\tstep - 1}_{\tstep - \ngr + 1}}}{\symt} \justification{\cref{lemma:hard-attention-lemma-2}.} \\
        & = \frac{\exp\left(\outMtx \; \onehot{\str_{\strlen - \ngr + 2 : \strlen}}\right)_\eos}{\sum_{\sym \in \eosalphabet} \exp\left(\outMtx \; \onehot{\str_{\strlen - \ngr + 2 : \strlen}}\right)_\sym} \prod_{\tstep = 1}^{\strlen} \frac{\exp\left({\outMtx \; \onehot{\str^{\tstep - 1}_{\tstep - \ngr + 1}}}\right)_{\symt}}{\sum_{\sym \in \eosalphabet} \exp\left(\outMtx \; \onehot{\str_{\strlen - \ngr + 1 : \tstep - 1}}\right)_\sym} \justification{Definition of $\softmax$.} \\
        & = \frac{\exp\left(\log{\pLNSM\left(\eos\mid\str^{\strlen}_{\tstep - \ngr + 2}\right)}\right)}{\sum_{\sym \in \eosalphabet} \exp\left(\log{\pLNSM\left(\sym\mid\str^{\strlen}_{\tstep - \ngr + 2}\right)}\right)} \prod_{\tstep = 1}^{\strlen} \frac{\exp\left(\log{\pLNSM\left(\sym_\tstep\mid\str^{\tstep - 1}_{\tstep - \ngr + 1}\right)}\right)}{\sum_{\sym \in \eosalphabet} \exp\left(\log{\pLNSM\left(\sym\mid\str^{\tstep - 1}_{\tstep - \ngr + 1}\right)}\right)} \justification{\cref{eq:def-outmtx}.} \\
        & = \pLNSM\left(\eos\mid\str_{\strlen - \ngr + 2 : \strlen}\right) \prod_{\tstep = 1}^{\strlen} \pLNSM\left(\sym_\tstep\mid\str^{\tstep - 1}_{\tstep - \ngr + 1}\right) = \pLM\left(\str\right) \justification{The \ngram LM is autoregressive.}
    \end{align}
    \end{subequations}
\end{proof}

\subsubsection{Simulation with \texorpdfstring{$\ngr - 1$}{n-1} Layers: The Intuition} 
This section presents the construction of a transformer LM with a \emph{single} head but $\ngr - 1$ layers that is weakly equivalent to a given \ngram LM. 
We again outline the intuition first before giving the technical details below as part of \cref{lemma:multi-layer-lemma}.
The construction we present resembles the one from the proof of \cref{thm:transformers-n-gram-label}.
The main difference is that, instead of using $\ngr - 1$ attention heads to identify the history, we instead use $\ngr - 1$ transformer \emph{layers}.
Intuitively, each of the $\ngr - 1$ layers iteratively adds one of the $\ngr - 1$ symbols needed to identify the history, resulting in the identification of the full history after the $\left(\ngr - 1\right)\textsuperscript{st}$ layer.
Once the history is identified, the $\left(\ngr - 1\right)\textsuperscript{st}$ layer can compute the conditional distribution over the next symbol as in the proof of \cref{thm:transformers-n-gram-label}.
This can be illustrated as the following sequence of transformations:\footnote{We leave out the positional encodings for clarity.}
\begin{subequations}
\begin{align}
    \tflayerinputmat^1       & \colon \begin{pmatrix}
      \onehot{\sym_1} & \onehot{\sym_2} & \onehot{\sym_3} & \cdots & \onehot{\sym_\tstep} & \cdots & \onehot{\sym_\strlen} \\
      \zero_\nsymbols & \zero_\nsymbols & \zero_\nsymbols & \cdots & \zero_\nsymbols & \cdots & \zero_\nsymbols       \\
      \zero_\nsymbols & \zero_\nsymbols & \zero_\nsymbols & \cdots & \zero_\nsymbols & \cdots & \zero_\nsymbols       \\
      \vdots          & \vdots          & \vdots          & \ddots & \vdots & \ddots & \vdots                \\
      \zero_\nsymbols & \zero_\nsymbols & \zero_\nsymbols & \cdots & \zero_\nsymbols & \cdots & \zero_\nsymbols       \\
                                      \end{pmatrix}                      \\
    \tflayerinputmat^2       & \colon \begin{pmatrix}
      \onehot{\sym_1} & \onehot{\sym_2} & \onehot{\sym_3} & \cdots & \onehot{\sym_\tstep} & \cdots & \onehot{\sym_\strlen}       \\
      \zero_\nsymbols & \onehot{\sym_1} & \onehot{\sym_2} & \cdots & \onehot{\sym_{\tstep - 1}} & \cdots & \onehot{\sym_{\strlen - 1}} \\
      \zero_\nsymbols & \zero_\nsymbols & \zero_\nsymbols & \cdots & \zero_\nsymbols & \cdots & \zero_\nsymbols       \\
      \vdots          & \vdots          & \vdots          & \ddots & \vdots & \ddots & \vdots                      \\
      \zero_\nsymbols & \zero_\nsymbols & \zero_\nsymbols & \cdots & \cdots & \cdots & \zero_\nsymbols             \\
  \end{pmatrix}           \\
                             & \vdots   \nonumber                                                                                                        \\
    \tflayerinputmat^{\ngr - 1} & \colon \begin{pmatrix}
  \onehot{\sym_1} & \onehot{\sym_2} & \onehot{\sym_3} & \cdots & \onehot{\sym_{\tstep}} & \cdots & \onehot{\sym_\strlen}                 \\
  \zero_\nsymbols & \onehot{\sym_1} & \onehot{\sym_2} & \cdots & \onehot{\sym_{\tstep - 1}} & \cdots & \onehot{\sym_{\strlen - 1}}           \\
  \zero_\nsymbols &\zero_\nsymbols & \onehot{\sym_1} & \cdots & \onehot{\sym_{\tstep - 2}} & \cdots & \onehot{\sym_{\strlen - 1}}           \\
      \vdots          & \vdots          & \vdots          & \ddots & \vdots & \ddots & \vdots                      \\
  \zero_\nsymbols & \zero_\nsymbols & \zero_\nsymbols & \cdots & \onehot{\sym_{\tstep - \ngr + 1}} & \cdots & \onehot{\sym_{\strlen - \ngr + 2}} \\
                                      \end{pmatrix}
\end{align}
\end{subequations}

Such representations can be constructed by starting with the initial symbol encoding ($\tflayerinputmat^1$) and then \emph{passing over} the information from the $\tstep\textsuperscript{th}$ symbol to the contextual representations of the $\left(\tstep + 1\right)\textsuperscript{st}$ symbol in $\tflayerinputmat^2$ by adding a \emph{shifted} the contextual representation of the $\tstep\textsuperscript{th}$ symbol.
This is performed $\ngr - 1$ times, resulting in contextual representations of the form $\tflayerinputmat^{\ngr - 1}$ where the $\tstep\textsuperscript{th}$ contextual representation contains the (ordered) information about the preceding $\ngr - 1$ symbols, i.e., the required history.
Intuitively, the $\ell\textsuperscript{th}$ transformation of the contextual representation $\vx_\tstep^{\ell - 1}$ should be of the form
\begin{equation}
    \vx_\tstep^{\ell} = \underbrace{\vx_\tstep^{\ell - 1}}_{\substack{\text{The previous $\ell - 1$} \\ \text{symbols.}}} + \underbrace{\downarrow \vx_{\tstep - 1}^{\ell - 1}}_{\substack{\text{The symbol $\ell$ symbols back} \\ \text{shifted one ``cell'' downward.}}}
\end{equation}
The first term is simply the residual connection of the transformer layer.
The second term is the shifted representation of the symbol $\sym_{\tstep - \ell}$, which can be performed by a simple linear transformation in the value transformation $\vTransf$.

\subsubsection{Simulation with \texorpdfstring{$\ngr - 1$}{n-1} Layers: Proofs} 
We now make the intuition presented in the previous section more formal.
Concretely, we only investigate how to identify the correct $\ngr - 1$ symbols in the history with the $\ngr - 1$ layers. 
We then rely on \cref{lemma:hard-attention-lemma-2} and the derivation from the proof of \cref{thm:transformers-n-gram-label} again to convert this information into a weakly equivalent transformer LM.
Let $\ell \in \NTo{\ngr - 1}$.
Define the following parameters of the attention head of the $\ell\textsuperscript{th}$ transformer layer:
\begin{equation}
    \posInEmbeddingFun{\sym, \tstep} \defeq \begin{pmatrix}
         \onehot{\sym}      \\
         \zero_\bosnsymbols \\
         \vdots             \\
         \zero_\bosnsymbols \\
          \sqrt{\frac{1}{\tstep}} \\
          \sqrt{1 - \frac{1}{\tstep}} \\
          \sqrt{\frac{1}{\tstep + 1}} \\
          \sqrt{1 - \frac{1}{\tstep + 1}} \\
     \end{pmatrix} \in \R^{\left(\ngr - 1\right) \bosnsymbols + 4},
\end{equation}
\begin{equation}
    \tfscorefun\left(\vq, \vk\right) \defeq \innerProd{\vq}{\vk},
\end{equation}
\begin{subequations}
\begin{align}
    \qTransf\left(\vx\right) & \defeq \mQ \vx, \quad \mQ \in \R^{2 \times \left(\left(\ngr - 1\right) \bosnsymbols + 4\right)} \\
    \kTransf\left(\vx\right) & \defeq \mK \vx, \quad \mK \in \R^{2 \times \left(\left(\ngr - 1\right) \bosnsymbols + 4\right)}                     \\
    \vTransf\left(\vx\right) & \defeq \mV \vx, \quad \mV \in \R^{\left(\left(\ngr - 1\right) \bosnsymbols + 4\right)\times \left(\left(\ngr - 1\right) \bosnsymbols + 4\right)}, \\
    \oTransf\left(\vx\right) & \defeq \mO \vx, \quad \mO \in \R^{\left(\left(\ngr - 1\right) \bosnsymbols + 4\right)\times \left(\left(\ngr - 1\right) \bosnsymbols + 4\right)},
\end{align}
\end{subequations}
\begin{subequations}
\begin{align}
    \mQ_{:, \left(\ngr - 1\right) \bosnsymbols + 1: \left(\ngr - 1\right) \bosnsymbols + 2}                                                                        & \defeq \mI_2, \\
    \mK_{:, \left(\ngr - 1\right) \bosnsymbols + 3: \left(\ngr - 1\right) \bosnsymbols + 4}                                                                        & \defeq \mI_2,                   \\
    \mV_{1 + \ell \bosnsymbols: 1 + \left(\ell + 1\right) \bosnsymbols, 1 + \left(\ell - 1\right) \bosnsymbols: 1 + \ell \bosnsymbols} & \defeq \mI_\bosnsymbols,
\end{align}
\end{subequations}
where the unspecified elements of $\mQ, \mK$, and $\mV$ are $0$.
Schematically, $\mV$ looks as follows:
\begin{equation}
    \mV = \begin{blockarray}{ccccc}
        {\scriptstyle \left(\ell - 1\right) \bosnsymbols} & {\scriptstyle \bosnsymbols} & {\scriptstyle \left(\ngr - 1 - \ell\right) \bosnsymbols} & {\scriptstyle 4} & \\
        \overbrace{\hspace{1.5cm}}^{\text{$\left(\ell - 1\right)$ symbol slots}} & \overbrace{\hspace{1.5cm}}^{\text{$\ell\textsuperscript{th}$ symbol slot}} & \overbrace{\hspace{1.5cm}}^{\text{$\left(\ngr - 1 - \ell\right)$ symbol slots}} & \overbrace{\hspace{1.5cm}}^{\text{$4$ position slots}} & \\
        \begin{block}{(cccc)c}
            & & & & {\scriptstyle \ell \bosnsymbols} \\
            & \mI_\bosnsymbols & & & {\scriptstyle \bosnsymbols} \\
            & & & & {\scriptstyle \left(\ngr - 1 - \left(\ell + 1\right)\right) \bosnsymbols} \\
            & & & & {\scriptstyle 4} \\
        \end{block}
    \end{blockarray}
\end{equation}
That is, $\mI_\bosnsymbols$ occupies the ``cell'' $\ell$ cells down and $\ell - 1$ right of the top-left corner.
Moreover, the matrix $\mO$ is a matrix of all zeros, meaning that $\oTransf\left(\vx\right) = \zero$ for all $\vx$.
Again, notice that the position-augmented symbol representation function $\posInEmbedding$ can be implemented by concatenating or summing a symbol- and a position-specific component.
\begin{lemma} \label{lemma:multi-layer-lemma}
    With the parameters defined above, it holds that
    \begin{equation}
        \vx^\ell_\tstepminus =            \begin{pmatrix}
            \onehot{\sym_\tstepminus}                                                        \\
            \onehot{\sym_{\tstepminus - 1}}                                                  \\
            \onehot{\sym_{\tstepminus - 2}}                                                  \\
            \vdots                                                                           \\
            \onehot{\sym_{\tstepminus - \ell}} \\
            \zero_\bosnsymbols                                                               \\
            \vdots                                                                           \\
            \zero_\bosnsymbols                                                               \\
          \sqrt{\frac{1}{\tstep}} \\
          \sqrt{1 - \frac{1}{\tstep}} \\
          \sqrt{\frac{1}{\tstep + 1}} \\
          \sqrt{1 - \frac{1}{\tstep + 1}} \\
        \end{pmatrix}
    \end{equation}
\end{lemma}
\begin{proof}
    We prove the lemma by induction.
    As in the proof of \cref{lemma:hard-attention-lemma-3}, we pad the input string with $\ngr - 1$ $\bos$ symbols to resolve the case when $\tstep - 1 < \ngr - 1$.

    \paragraph{Base Case.} 
    For $\ell = 1$, the claim follows from the definition of the symbol representation function $\posInEmbedding$.

    \paragraph{Inductive Step.}
    For $\ell > 1$, we assume that the claim holds for $\ell - 1$.
    We then prove that it holds for $\ell$ as well.
    By the construction of the keys and values matrices, as in the proof of \cref{lemma:hard-attention-lemma-3}, it holds that the attention head puts all its attention for query $\vq^\ell_{\tstep - 1}$ on the key $\vk^{\ell}_{\tstep - 2}$.
    This means that
    \begin{equation} \label{eq:reccurrence}
        \va^\ell_\tstepminus = \vx^{\ell - 1}_\tstepminus + \vv^\ell_{\tstep - 2}.
    \end{equation}

    By the induction hypothesis, we have that
    \begin{figure}[h!]
    \begin{minipage}{0.5\linewidth}
    \begin{equation}
        \vx^{\ell - 1}_\tstepminus = \begin{pmatrix}
            \onehot{\sym_\tstepminus}                                                                         \\
            \onehot{\sym_{\tstep - 2}}                                                                        \\
            \onehot{\sym_{\tstep - 3}}                                                                        \\
            \vdots                                                                                            \\
            \onehot{\sym_{\tstepminus - \left(\ell - 1\right)}} \\
            \zero_\bosnsymbols                                                                                \\
            \vdots                                                                                            \\
            \zero_\bosnsymbols                                                                                \\
          \sqrt{\frac{1}{\tstep}} \\
          \sqrt{1 - \frac{1}{\tstep}} \\
          \sqrt{\frac{1}{\tstep + 1}} \\
          \sqrt{1 - \frac{1}{\tstep + 1}} \\
        \end{pmatrix}
    \end{equation}
    \end{minipage}%
    \begin{minipage}{0.5\linewidth}
    \begin{equation}
        \vx^{\ell - 1}_{\tstep - 2} = \begin{pmatrix}
            \onehot{\sym_{\tstep - 2}}                                                      \\
            \onehot{\sym_{\tstep - 3}}                                                      \\
            \onehot{\sym_{\tstep - 4}}                                                      \\
            \vdots                                                                          \\
            \onehot{\sym_{\tstep - 2 - \left(\ell - 1\right)}} \\
            \zero_\bosnsymbols                                                              \\
            \vdots                                                                          \\
            \zero_\bosnsymbols                                                              \\
          \sqrt{\frac{1}{\tstep}} \\
          \sqrt{1 - \frac{1}{\tstep}} \\
          \sqrt{\frac{1}{\tstep + 1}} \\
          \sqrt{1 - \frac{1}{\tstep + 1}} \\
        \end{pmatrix}.
    \end{equation}
    \end{minipage}
    \end{figure}
    
    Furthermore, by definition of $\mV$, we have that
    \begin{subequations}
    \begin{align}
        \vv^\ell_{\tstep - 2} & = \vTransf\left(\vx^{\ell - 1}_{\tstep - 2}\right) \\
                              & = \mV \vx^{\ell - 1}_{\tstep - 2}                  \\
                              & =
        \begin{pmatrix}
            \zero_\bosnsymbols                                                              \\
            \vdots                                                                          \\
            \zero_\bosnsymbols                                                              \\
            \onehot{\sym_{\tstep - 2 - \left(\ell - 1\right)}} \\
            \zero_\bosnsymbols                                                              \\
            \vdots                                                                          \\
            \zero_\bosnsymbols                                                              \\
          \sqrt{\frac{1}{\tstep}} \\
          \sqrt{1 - \frac{1}{\tstep}} \\
          \sqrt{\frac{1}{\tstep + 1}} \\
          \sqrt{1 - \frac{1}{\tstep + 1}} \\
        \end{pmatrix}
    \end{align}
    \end{subequations}

    Inserting this into \cref{eq:reccurrence}, we get that $\va^\ell_{\tstep - 1}$ satisfies the required equality:
    \begin{subequations}
        \begin{align}
            \va^\ell_\tstepminus 
            &= \vx^{\ell - 1}_\tstepminus + \vv^\ell_{\tstep - 2} \\
            &= \begin{pmatrix}
            \onehot{\sym_\tstepminus}                                                                         \\
            \vdots                                                                                            \\
            \onehot{\sym_{\tstepminus - \left(\ell - 1\right)}} \\
            \zero_\bosnsymbols                                                                                \\
            \vdots                                                                                            \\
            \zero_\bosnsymbols                                                                                \\
          \sqrt{\frac{1}{\tstep}} \\
          \sqrt{1 - \frac{1}{\tstep}} \\
          \sqrt{\frac{1}{\tstep + 1}} \\
          \sqrt{1 - \frac{1}{\tstep + 1}} \\
        \end{pmatrix} + \begin{pmatrix}
            \zero_\bosnsymbols                                                              \\
            \vdots                                                                          \\
            \zero_\bosnsymbols                                                              \\
            \onehot{\sym_{\tstep - 2 - \left(\ell - 1\right)}} \\
            \zero_\bosnsymbols                                                              \\
            \vdots                                                                          \\
            \zero_\bosnsymbols                                                              \\
          \sqrt{\frac{1}{\tstep}} \\
          \sqrt{1 - \frac{1}{\tstep}} \\
          \sqrt{\frac{1}{\tstep + 1}} \\
          \sqrt{1 - \frac{1}{\tstep + 1}} \\
        \end{pmatrix} 
        = \begin{pmatrix}
            \onehot{\sym_\tstepminus}                                                                         \\
            \vdots                                                                                            \\
            \onehot{\sym_{\tstepminus - \left(\ell - 1\right)}} \\
            \onehot{\sym_{\tstep - 2 - \left(\ell - 1\right)}} \\
            \zero_\bosnsymbols                                                                                \\
            \vdots                                                                                            \\
            \zero_\bosnsymbols                                                                                \\
          \sqrt{\frac{1}{\tstep}} \\
          \sqrt{1 - \frac{1}{\tstep}} \\
          \sqrt{\frac{1}{\tstep + 1}} \\
          \sqrt{1 - \frac{1}{\tstep + 1}} \\
        \end{pmatrix}
        =  \begin{pmatrix}
            \onehot{\sym_\tstepminus}                                                                         \\
            \vdots                                                                                            \\
            \onehot{\sym_{\tstepminus - \left(\ell - 1\right)}} \\
            \onehot{\sym_{\tstepminus - \ell}} \\
            \zero_\bosnsymbols                                                                                \\
            \vdots                                                                                            \\
            \zero_\bosnsymbols                                                                                \\
          \sqrt{\frac{1}{\tstep}} \\
          \sqrt{1 - \frac{1}{\tstep}} \\
          \sqrt{\frac{1}{\tstep + 1}} \\
          \sqrt{1 - \frac{1}{\tstep + 1}} \\
        \end{pmatrix}
        \end{align}
    \end{subequations}
    Since the computation of $\vx^\ell_{\tstep - 1} \defeq \vz^\ell_{\tstep - 1} = \oTransfFun{\va^\ell_\tstepminus} + \va^\ell_\tstepminus = \va^\ell_\tstepminus$ with the definition of $\oTransf$ as the zero function gives us $\vz^\ell_{\tstep - 1} = \va^\ell_{\tstep - 1}$, we get the required equality, which finishes the proof.
\end{proof}

This allows us to prove the following theorem.
\kLayersSingleHeadTheorem*
\begin{proof}
    \cref{lemma:multi-layer-lemma} shows that the $\ngr - 1$-layer transformer can identify the history of interest. 
    Applying \cref{lemma:hard-attention-lemma-2} and the same derivation as in the proof of \cref{thm:transformers-n-gram-label} shows that we can construct a weakly equivalent hard attention transformer.
\end{proof}

\subsubsection{Simulation with a Single Layer and a Single Head: Intuition}

While the construction presented here is considerably less intuitive than that of \cref{thm:transformers-n-gram-label}, the steps of the proof remain the same---they include the identification of the individual symbols and their positions in the history, combining them into the one-hot encoding of the entire history, and using that to compute the correct next-symbol conditional distributions.
This proof focuses on encoding the entire history of interest $\str_{\tstep - \ngr + 1}^{\tstep - 1}$ into a single vector in a way that can be decoded to index the conditional probability distribution as in \cref{def:transformer-plnsm}.
This can then be used to index the appropriate conditional probability distributions as in the proof of \cref{thm:transformers-n-gram-label}.

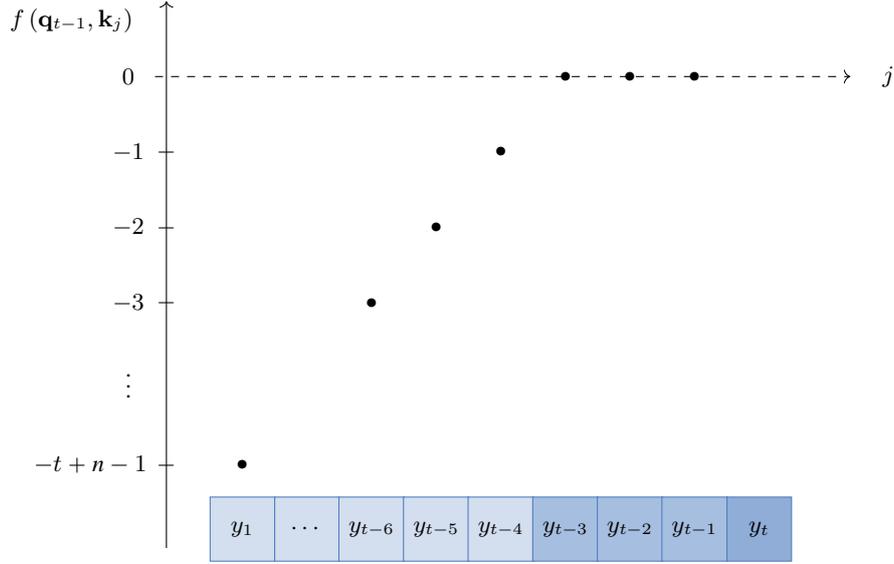
\begin{figure}
    \centering
    \begin{tikzpicture}[
        tape node/.style={draw=ETHBlue!80,minimum size=0.85cm,fill=ETHBlue!20},
        head node/.style={circle,minimum size=0.75cm,text=white},
        attn arrow/.style={-{Latex[length=3mm,width=2mm]},ETHGreen!100},
        comb arrow/.style={-{Latex[length=3mm,width=2mm]},ETHRed!70},
        comb node/.style={draw=ETHRed!80,circle,minimum size=1cm,fill=ETHRed!40},
        ]

        \foreach \i/\y in {0/$\sym_1$,1/$\cdots$,2/$\sym_{\tstep-6}$,3/$\sym_{\tstep-5}$,4/$\sym_{\tstep-4}$,5/$\sym_{\tstep-3}$,6/$\sym_{\tstep-2}$,7/$\sym_{\tstep-1}$,8/$\symt$} {
                \ifnum \i=8
                    \node[tape node,fill=ETHBlue!50] (tape-\i) at (0.85*\i,-1) {\footnotesize \y};
                \else
                    \ifnum \i>4
                        \ifnum \i<8
                            \node[tape node,fill=ETHBlue!40] (tape-\i) at (0.85*\i,-1) {\footnotesize \y};
                        \fi
                    \else
                        \node[tape node,fill=ETHBlue!20] (tape-\i) at (0.85*\i,-1) {\footnotesize \y};
                    \fi
                    \ifnum \i>8
                        \node[tape node,fill=ETHBlue!20] (tape-\i) at (0.85*\i,-1) {\footnotesize \y};
                    \fi
                \fi
            }

        \draw[->] (-1, -1.25) -- (-1, 6);
        \draw[->, dashed] (-1.15, 5) -- (8, 5);

        \node at (8.5, 5) {\footnotesize $\idxj$};
        \node at (-2.25, 5.75) {\footnotesize $\func\left(\vq_{\tstep - 1}, \vk_\idxj\right)$};

        \node at (-1.5, 5) {\footnotesize $0$};
        \node at (-1.5, 4) {\footnotesize $-1$};
        \node at (-1.5, 3) {\footnotesize $-2$};
        \node at (-1.5, 2) {\footnotesize $-3$};
        \node at (-1.5, 1) {\footnotesize $\vdots$};
        \node at (-2, -0.15) {\footnotesize $-\tstep + \ngr - 1$};

        \node at (-1, 2) {\footnotesize $-$};
        \node at (-1, 3) {\footnotesize $-$};
        \node at (-1, 4) {\footnotesize $-$};
        \node at (-1, -0.15) {\footnotesize $-$};

        \node at (0, -0.15) {\footnotesize $\bullet$};
        \node at (1.7, 2) {\footnotesize $\bullet$};
        \node at (2.55, 3) {\footnotesize $\bullet$};
        \node at (3.4, 4) {\footnotesize $\bullet$};

        \node at (4.25, 5) {\footnotesize $\bullet$};
        \node at (5.1, 5) {\footnotesize $\bullet$};
        \node at (5.95, 5) {\footnotesize $\bullet$};

    \end{tikzpicture}
    \caption{An illustration of the attention scores in a single-layer-single-head transformer network for a history $\str_{\tstep - 3 : \tstep - 1} = \sym_{\tstep - 3} \sym_{\tstep - 2} \sym_{\tstep - 1}$ when determining the conditional distribution $\pLM\left(\symt\mid\strlt \right)$.}
    \label{fig:single-attention-scores}
\end{figure}

Again, we outline the intuition first.
Fix $\tstep \leq |\str|$.
We decompose the computation of the representation $\tfencfun\left(\strlt\right)$ constructed by a single-layer-single-head transformer network as follows.
\begin{enumerate}
    \item Attending exactly to the history of interest with a single head and a single layer.
    This can be done by assigning the same attention \emph{score} with the scoring function to all positions within the history $\str_{\tstep - \ngr + 1}^{\tstep}$ and a lower score to all other positions.
    Keeping the definitions of $\qTransf$ and $\vTransf$ from \cref{thm:transformers-n-gram-label}, we can achieve that by defining
    \begin{equation}
        \tfscorefun\left(\vq_{\tstep - 1}, \vk_\idxj\right) = -\ReLUfunc{\innerProd{\vq_{\tstep - 1}}{\vk_\idxj}}
    \end{equation}
    which assigns positions in the history score $0$ and others negative scores.
    This is illustrated in \cref{fig:single-attention-scores}.
    Concretely, the scoring function $\tfscorefun$ together with hard attention results in attention weights of the form
    \begin{equation} \label{eq:single-head-single-layer-weights}
        \evs_\idxj = \frac{1}{\ngr - 1} \ind{\idxj \geq \tstep - \ngr + 1}.
    \end{equation}
    \item Storing the order of the symbols in the history.
    We will define the position-augmented representations as position-scaled one-hot encodings of the input symbols.
    In particular, define
    \begin{equation}
        \posInEmbedding\left(\symtminus, \idxj\right) \defeq\begin{pmatrix}
             10^{-\idxj} \cdot \onehot{\symtminus} \\ \tstep  \\ 1
        \end{pmatrix}  \in \R^{\bosnsymbols + 2}
    \end{equation}
    This effectively stores the information about both the position of the current symbol (with the magnitude) as well as the identity of the symbol $\sym_\idxj$ (with the index of the non-zero entry).
    Unlike in the multi-head or multi-layer case, note that in this case, the function $\posInEmbedding$ is not a concatenation (or addition) of two separate components (one for symbols and one for positions).
\end{enumerate}

Ignoring residual connections, \cref{eq:single-head-single-layer-weights} then implies that 
\begin{equation}
    \va_{\tstep - 1} = \sum_{\idxj = \tstep - \ngr + 1}^{\tstep - 1} \frac{1}{\ngr - 1} 10^{-\idxj} \onehot{\sym_{\idxj}}.
\end{equation}
For simplicity, we now write $\va \defeq \va_{\tstep - 1}$.
The individual entries of $\va$ will correspond to symbols in $\bosalphabet$.
The \emph{digits} of these symbols will encode the positions of the symbols in the history.
Concretely, $\va$ will have a non-zero digit in the $\idxi\textsuperscript{th}$ position if and only if the symbol $\sym$ appears in the string $\strlt$ at position $\idxi$ for $\idxi \in \NTo{\tstep - 1}$.
For example, in a $5$-gram LM over the alphabet $\alphabet = \set{\syma, \symb}$, the contextual representation $\va$ for the history $\str_{\tstep - 4 : \tstep - 1} = \syma \symb \syma \syma$ will be
\begin{equation}
    \va = \frac{1}{4}\begin{pmatrix}
        10^{\tstep - 1} + 10^{\tstep - 3} + 10^{\tstep - 4} \\
        10^{\tstep - 2}
    \end{pmatrix}
    = \frac{1}{4} 10^{-\tstep}
    \begin{pmatrix}
        0.1011 \\
        0.0100
    \end{pmatrix}.
\end{equation}
Such representations therefore uniquely encode the history of interest.

\paragraph{Decoding the representations of the history.}
The representations $\va$, therefore, contain the information about the history of interest compactly represented in a $\bosnsymbols$-dimensional vector $\va$. 
We now want to construct a function that transforms the constructed vector $\va$ into a one-hot encoding of the history.
To make $\va$ invariant with respect to $\tstep$, we first scale it by $\frac{\ngr - 1}{10^{\ngr - \tstep}}$ and define
\begin{equation}
    \va' = \frac{\ngr - 1}{10^{\ngr - \tstep}} \va.
\end{equation}
Then
\begin{equation}
    \eva'_\sym = \sum_{\idxi = 1}^{\ngr - 1} \ind{\sym = \sym_{\idxi + \tstep - \ngr}} 10^{- \idxi}.
\end{equation}
In words, the $\bosnsymbols$ entries of $\va'$ are of the form $\eva'_\sym = 0.d_1 \ldots d_{\ngr - 1}$, where $d_i = 1$ if and only if $\sym$ appears in the history $\strlt$ at position $\tstep - \ngr + i$.

We now focus on a specific symbol $\sym \in \bosalphabet$ with the entry $\eva'_\sym$ in the vector $\va'$ and decode it into a vector that can be used to construct a one-hot encoding of the relevant history $\str_{\tstep - \ngr + 1}^{\tstep - 1}$ with $\fTransf$.
The ``decoding function'' will take the form of a $\ngr - 1$-layer $\ReLU$-activated MLP.
Intuitively, each of the $\ngr - 1$ layers will contain $\bosnsymbols$ neurons, where each of them will compute the values $d_\idxi$ for a particular $\sym \in \alphabet$.
Now, $d_1$ can be computed as
\begin{equation}
    d_1 = \ind{10^1 \cdot \eva_\sym - 1 + 10^{-\ngr} > 0}.
\end{equation}
Since $\ind{\cdot}$ is not a continuous function, and, unlike in \cref{sec:and}, the arguments can now take arbitrary real values, it cannot be implemented using a composition of $\ReLU$ functions.
We can, however, simulate the discontinuous function with a linear combination of two $\ReLU$ functions that together define the same function as $\ind{\cdot}$ on a subset of $\R$ relevant for our purposes.
Notice that, as long as $d_1 = 1$, we have that $10^1 \cdot \eva_\sym - 1 + 10^{-\ngr} \geq 10^{-\ngr}$ while we have $10^1 \cdot \eva_\sym - 1 + 10^{-\ngr + 1} = 10^1 \cdot 0 - 1 + 10^{-\ngr} < 0$ if $d_1 = 0$.
This means that our approximation of $\ind{\cdot}$ only has to map values greater or equal to $10^{-\ngr}$ to $1$, rather than all positive values.
This allows us to continuously transition from $0$ to $1$ as the input to the $\ReLU$ function increases from $0$ to $10^{-\ngr}$.
Such a piecewise linear approximation of $\ind{\cdot}$ can be easily implemented by a linear combination of $\ReLU$ functions, i.e., with an MLP.
See \cref{fig:ind-approx} for an illustration of the approximation.
\begin{figure}
    \centering
    \begin{tikzpicture}
        \begin{axis}[
                legend pos=north west,
                axis x line=middle,
                axis y line=middle,
                xtick={2}, 
                xticklabels={$\scriptstyle 10^{-\ngr}$}, 
                ytick={0,1},
                yticklabels={0,1},
                grid = major,
                width=0.75\textwidth,
                height=4cm,
                grid=none,
                xmin=-10,     
                xmax= 10,    
                ymin= -0.1,     
                ymax= 1.4,   
                xlabel=$x$,
                ylabel=$y$,
                tick align=outside,
                enlargelimits=false]
            \addplot[domain=-10:0, DarkOrange, thick, samples=500] {-0.01};
            \addplot[domain=-10:0, AccentBlue, thick,samples=500] {0.01};
            \addplot[domain=0:10, DarkOrange, thick, samples=500] {1.01};
            \addplot[domain=2:10, AccentBlue, thick,samples=500] {0.99};
            \addplot[domain=0:2, AccentBlue, thick,samples=500] {x/2};
        \end{axis}
    \end{tikzpicture}
    \caption{The \textcolor{DarkOrange}{step} function $\ind{x}$ and the \textcolor{AccentBlue}{approximated step} function that matches $\ind{x}$ outside of $\left[0, 10^{-\ngr}\right]$.}
    \label{fig:ind-approx}
\end{figure}

$d_2$ can then be computed as
\begin{equation}
    d_2 = \ind{10^2 \cdot \eva_\sym - 10^1 d_1 - 1 + 10^{-\ngr} > 0}
\end{equation}
and, in general,
\begin{equation}
    d_\idxi = \ind{10^\idxi \cdot \eva_\sym - \sum_{\idxj = 1}^{\idxi - 1} 10^{\idxi - \idxj} d_\idxj - 1 + 10^{-\ngr} > 0}.
\end{equation}
The computation of $d_\idxi$ requires $\idxi$ layers (since $d_{\idxj}$ for $\idxj < \idxi$ have to be computed first), meaning that $\ngr - 1$ layers are required in total.
Altogether, these layers compute the values $d_\idxi$ for a single $\sym \in \alphabet$.
Replicating this computation for every $\sym \in \alphabet$ and concatenating the results gives us the desired contextual representation $\vz$.

With this construction, it holds for every $\sym \in \alphabet$ that $d_\idxi = 1$ if and only if the symbol $\sym$ appears in the history $\str^{\tstep - 1}_{\tstep - \ngr + 1}$ at position $\tstep - \ngr + \idxi$.
This, therefore, gives us enough information to reconstruct the multi-hot encoding of the history of interest.
As in \cref{thm:transformers-n-gram-label}, this can then be converted into a one-hot encoding using another $\ReLU$ layer to implement the logical \texttt{AND} operation.
This intuition is formalized in the following section.

\subsubsection{Simulation with a Single Layer and a Single Head: Proofs}

We define the following parameters of the transformer head.\footnote{For simplicity, we ignore residual connections in this section since we do not require them and the omission makes the presentation cleaner. By duplicating the representations as in \cref{thm:transformers-n-gram-label}, residual connections could easily be added back to make this setting closer to the general transformer setting.}
\begin{itemize}
    \item Static encodings
    \begin{equation} \label{eq:single-static}
        \posInEmbedding\left(\symtminus, \idxj\right) \defeq\begin{pmatrix}
             10^{-\idxj} \cdot \onehot{\symtminus} \\ 1 \\ \tstep
        \end{pmatrix} \in \R^{\bosnsymbols + 3}
    \end{equation}
    \item Query, key, value, and output transformations
    \begin{subequations}
    \begin{alignat}{2}
        \qTransf\left(\vx\right) & \defeq \mQ \vx + \vb_\qTransf, \qquad &&\mQ \in \R^{2 \times \left(\bosnsymbols + 2\right)}, \vb_\qTransf \in \R^2,            \\
        \kTransf\left(\vx\right) & \defeq \mK \vx, \qquad &&\mK \in \R^{2 \times \left(\bosnsymbols + 2\right)},                                \\
        \vTransf\left(\vx\right) & \defeq \mV \vx, \qquad &&\mV \in \R^{\bosnsymbols \times \left(\bosnsymbols + 2\right)},                              \\
        \oTransf\left(\vx\right) & \defeq \mI_\bosnsymbols \vx,
    \end{alignat}
    \end{subequations}
    \begin{subequations}
    \begin{align}
        \mQ_{:, \bosnsymbols + 1: \bosnsymbols + 2}             & \defeq \begin{pmatrix}
            0 & 1 \\ -1 & 0
        \end{pmatrix},  \\
        \vb_{\qTransf}             & \defeq \begin{pmatrix}
            -\left(\ngr - 2\right) \\ 0
        \end{pmatrix} \\                 
        \mK_{:, \bosnsymbols + 1: \bosnsymbols + 2}             & \defeq \mI_2 \\
        \mV_{:, 1: \bosnsymbols} & \defeq \mI_\bosnsymbols \label{eq:single-value}
    \end{align}
    \end{subequations}
    where the unspecified elements of $\mQ, \mK$, $\mV$, $\mO$ are $0$.
\end{itemize}

A large part of the proof of correctness will rely on identifying the digits of the contextual representations of strings. 
We will rely heavily on the following definition.
\begin{definition}
    Let $x \in \Q \cap \left(10^{-\left(N + 1\right)}, 10^{-1}\right]$ be a rational-valued number with at most $N$ digits in its decimal representation.
    We define $\digit{x}{\idxi}$ as the $\idxi\textsuperscript{th}$ digit of $x$ for $\idxi \in \NTo{N}$.
    We also group these $N$ values into the vector $\digitVec{x}$:
    \begin{equation}
        \digitVec{x} \defeq \begin{pmatrix}
            \digit{x}{1} \\
            \vdots \\
            \digit{x}{N}
        \end{pmatrix}.
    \end{equation}
\end{definition}

\begin{lemma} \label{lem:digit-0}
    A transformer with the parameters and functions defined in \cref{eq:single-static}--\cref{eq:single-value} computes for string $\str \in \kleene{\alphabet}$
    \begin{equation}
        \va_{\tstep - 1} = \sum_{\idxj = \tstep - \ngr + 1}^{\tstep - 1} \frac{1}{\ngr - 1} 10^{-\idxj} \onehot{\sym_{\idxj}}.
    \end{equation}
\end{lemma}
\begin{proof}
    By construction, we have
    \begin{subequations}
    \begin{align}
        \vq_{\tstep - 1} &= \begin{pmatrix}
            \tstep - 1 - \left(\ngr - 2\right) \\
            -1
        \end{pmatrix} \\
        \vk_\idxj &= \begin{pmatrix}
            1 \\ 
            \idxj
        \end{pmatrix} \\
        \vv_\idxj &= 10^{-\idxj} \cdot \onehot{\symtminus},
    \end{align}
    \end{subequations}
    thus
    \begin{subequations}
    \begin{align}
        \tfscorefun\left(\vq_{\tstep - 1}, \vk_\idxj\right) 
        &= -\ReLUfunc{\innerProd{\vq_{\tstep - 1}}{\vk_\idxj}} \\
        &= -\ReLUfunc{\innerProd{\begin{pmatrix}
            \tstep - 1 - \left(\ngr - 2\right) \\
            -1
        \end{pmatrix}}{\begin{pmatrix}
            1 \\ 
            \idxj
        \end{pmatrix}}} \\
        &= -\ReLUfunc{\tstep - 1 - \ngr + 2 - \idxj} \\
        &= -\ReLUfunc{\tstep - \ngr + 1 - \idxj} \\
        &= \begin{cases}
            0 & \ifcondition \idxj \geq \tstep - \ngr + 1 \\
            < 0 &\otherwisecondition
        \end{cases}
    \end{align}
    \end{subequations}
    meaning that
    \begin{equation}
        \evs_\idxj = \frac{1}{\ngr - 1} \ind{\idxj \geq \tstep - \ngr + 1}.
    \end{equation}
    This means that 
    \begin{equation}
        \va_{\tstep - 1} = \sum_{\idxj = \tstep - \ngr + 1}^{\tstep - 1} \frac{1}{\ngr - 1} \vv_{\idxj} = \sum_{\idxj = \tstep - \ngr + 1}^{\tstep - 1} \frac{1}{\ngr - 1} 10^{-\idxj} \onehot{\sym_{\idxj}}
    \end{equation}
    which is what we needed to prove.
\end{proof}
\begin{lemma} \label{lem:digit-1}
    Define
    \begin{equation}
        \va' \defeq \frac{\ngr - 1}{10^{\ngr - \tstep}} \; \va_{\tstep - 1}
    \end{equation}
    with $\va_{\tstep - 1}$ from \cref{lem:digit-0}.
    Indexing the $\bosnsymbols$ elements of $\va'$ directly with $\sym \in \bosalphabet$, it holds that
    \begin{equation}
        \digit{\eva'_\sym}{\idxi} = 1 \iff \sym_{\tstep - \ngr + \idxi} = \sym
    \end{equation}
    for all $\idxi \in \NTo{N}$ and $\digit{x}{\idxi} = 0$ for all $\idxi > N$.
\end{lemma}
\begin{proof}
    By \cref{lem:digit-0}, $\va$ contains entries of the form
    \begin{subequations}
    \begin{align}
        \eva_\sym &= 
        \frac{1}{\ngr - 1} \sum_{\idxj = \tstep - \ngr + 1}^{\tstep - 1} \ind{\sym = \sym_\idxj} 10^{-\idxj} \\
        &= \frac{1}{\ngr - 1} \sum_{\idxj' = 1}^{\tstep - 1 - \left(\tstep - \ngr\right)} \ind{\sym = \sym_{\idxj' + \tstep - \ngr}} 10^{-\left(\idxj' + \tstep - \ngr\right)} \justification{Change of variables.} \\        
        &= \frac{1}{\ngr - 1} \sum_{\idxj = 1}^{\ngr - 1} \ind{\sym = \sym_{\idxj + \tstep - \ngr}} 10^{\ngr - \tstep - \idxj} \justification{Change of variables.} \\        
        &= \frac{10^{\ngr - \tstep}}{\ngr - 1} \sum_{\idxj = 1}^{\ngr - 1} \ind{\sym = \sym_{\idxj + \tstep - \ngr}} 10^{- \idxj} \justification{Distributivity.}
    \end{align}
    \end{subequations}
    for $\sym \in \bosalphabet$.
    Then
    \begin{subequations}
    \begin{align}
        \eva'_\sym &\defeq \frac{\ngr - 1}{10^{\ngr - \tstep}} \eva_\sym \\
        &= \frac{\ngr - 1}{10^{\ngr - \tstep}} \frac{10^{\ngr - \tstep}}{\ngr - 1} \sum_{\idxj = 1}^{\ngr - 1} \ind{\sym = \sym_{\idxj + \tstep - \ngr}} 10^{- \idxj} \\
        &= \sum_{\idxj = 1}^{\ngr - 1} \ind{\sym = \sym_{\idxj + \tstep - \ngr}} 10^{- \idxj}.
    \end{align}
    \end{subequations}
    This implies that $\digit{\eva'_\sym}{\idxi} = 1 \iff \sym_{\tstep - \ngr + \idxi} = \sym$ for $\idxi \in \NTo{N}$ and that $\digit{x}{\idxi} = 0$ for $\idxi > N$, which is what we wanted to prove.
\end{proof}

\begin{lemma} \label{lem:digit-equality}
Let $x \in \Q \cap \left(10^{-\left(N + 1\right)}, 10^{-1}\right]$ with $\digit{x}{\idxi} \in \set{0, 1}$ for $\idxi \in \NTo{N}$. 
Then, $\digit{x}{\idxi}$ satisfy the equality
\begin{equation} \label{eq:r-eq}
    \digit{x}{\idxi} = \ind{10^\idxi \cdot x - \sum_{\idxj = 1}^{\idxi - 1} 10^{\idxi - \idxj} \; \digit{x}{\idxj} - 1 + 10^{-\left(N + 1\right)} > 0}.
\end{equation}
for all $\idxi \in \NTo{N}$.
\end{lemma}
\begin{proof}
    Let $x \in \Q \cap \left(10^{-\left(N + 1\right)}, 10^{-1}\right]$ with $\digit{x}{\idxi} \in \set{0, 1}$ for $\idxi \in \NTo{N}$. 
    Then, by definition of $\digit{x}{\idxi}$, we have that
    \begin{equation}
        x = \sum_{\idxj = 1}^N \digit{x}{\idxj} 10^{-\idxj}.
    \end{equation}
    This means that
    \begin{subequations}
    \begin{align}
        10^\idxi x 
        &= 10^\idxi \sum_{\idxj = 1}^N \digit{x}{\idxj} 10^{- \idxj} \\
        &= \sum_{\idxj = 1}^N \digit{x}{\idxj} 10^{\idxi - \idxj} \\
        &= \sum_{\idxj = 1}^{\idxi - 1} \digit{x}{\idxj} 10^{\idxi - \idxj} + \digit{x}{\idxi} + \sum_{\idxj = \idxi + 1}^{N} \digit{x}{\idxj} 10^{\idxi - \idxj},
    \end{align}
    \end{subequations}
    implying that
    \begin{equation}
        10^\idxi \cdot x - \sum_{\idxj = 1}^{\idxi - 1} 10^{\idxi - \idxj} \; \digit{x}{\idxj} = \digit{x}{\idxi} + \sum_{\idxj = \idxi + 1}^{N} \digit{x}{\idxj} 10^{\idxi - \idxj}.
    \end{equation}
    
    Suppose now that $\digit{x}{\idxi} = 1$. 
    Then 
    \begin{subequations}
    \begin{align}
        10^\idxi \cdot x - \sum_{\idxj = 1}^{\idxi - 1} 10^{\idxi - \idxj} \; \digit{x}{\idxj} - 1 + 10^{-\left(N + 1\right)} 
        &= \digit{x}{\idxi} + \sum_{\idxj = \idxi + 1}^{N} \digit{x}{\idxj} 10^{\idxi - \idxj}  - 1 + 10^{-\left(N + 1\right)} \\
        &= 1 + \sum_{\idxj = \idxi + 1}^{N} \digit{x}{\idxj} 10^{\idxi - \idxj} - 1 + 10^{-\left(N + 1\right)} \\
        &= \sum_{\idxj = \idxi + 1}^{N} \digit{x}{\idxj} 10^{\idxi - \idxj} + 10^{-\left(N + 1\right)} > 0,
    \end{align}
    \end{subequations}
    meaning that 
    \begin{equation}
        \ind{10^\idxi \cdot x - \sum_{\idxj = 1}^{\idxi - 1} 10^{\idxi - \idxj} \; \digit{x}{\idxj} - 1 + 10^{-\left(N + 1\right)}} = 1 = \digit{x}{\idxi}.
    \end{equation}

    Suppose, on the contrary, that $\digit{x}{\idxi} = 0$. 
    Then 
    \begin{subequations}
    \begin{align}
        10^\idxi \cdot x - \sum_{\idxj = 1}^{\idxi - 1} 10^{\idxi - \idxj} \; \digit{x}{\idxj} - 1 + 10^{-\left(N + 1\right)} 
        &= \digit{x}{\idxi} + \sum_{\idxj = \idxi + 1}^{N} \digit{x}{\idxj} 10^{\idxi - \idxj}  - 1 + 10^{-\left(N + 1\right)} \\
        &= 0 + \sum_{\idxj = \idxi + 1}^{N} \digit{x}{\idxj} 10^{\idxi - \idxj} - 1 + 10^{-\left(N + 1\right)} \\
        &= \sum_{\idxj = \idxi + 1}^{N} \digit{x}{\idxj} 10^{\idxi - \idxj} + 10^{-\left(N + 1\right)} - 1 \\
        &= \sum_{\idxj' = 1}^{N - \left(\idxi + 1\right)} \digit{x}{\idxj' + \idxi + 1} 10^{\idxi - \idxj' - \idxi - 1} + 10^{-\left(N + 1\right)} - 1 \\
        &= \underbrace{\sum_{\idxj' = 1}^{N - \left(\idxi + 1\right)} \digit{x}{\idxj' + \idxi + 1} 10^{- \idxj' - 1} + 10^{-\left(N + 1\right)}}_{< 1} - 1 < 0 
    \end{align}
    \end{subequations}
    meaning that 
    \begin{equation}
        \ind{10^\idxi \cdot x - \sum_{\idxj = 1}^{\idxi - 1} 10^{\idxi - \idxj} \; \digit{x}{\idxj} - 1 + 10^{-\left(N + 1\right)}} = 0 = \digit{x}{\idxi}.
    \end{equation}
    This finishes the proof.
\end{proof}

\begin{lemma} \label{lemma:single-lemma-2}
    Let $x \in \Q \cap \left(10^{-\left(N + 1\right)}, 10^{-1}\right]$ with $\digit{x}{\idxi} \in \set{0, 1}$ for $\idxi \in \NTo{N}$. 
    Given $\digit{x}{\idxj}$ for $\idxj < \idxi$, $\digit{x}{\idxi}$ can be computed by a single layer MLP.
\end{lemma}
\begin{proof}
    By \cref{lem:digit-equality}, we can use the knowledge of $\digit{x}{\idxj}$ for $\idxj < \idxi$, $\digit{x}{\idxi}$ to implement the function
    \begin{equation} \label{eq:r-eq-again}
        \digit{x}{\idxi} = \ind{10^\idxi \cdot x - \sum_{\idxj = 1}^{\idxi - 1} 10^{\idxi - \idxj} \; \digit{x}{\idxj} - 1 + 10^{-\left(N + 1\right)} > 0}.
    \end{equation}
    with an MLP.
    For any $\idxi \in \NTo{N}$, the inner function 
    \begin{subequations}
    \begin{align}
        \begin{pmatrix}
            \digit{x}{1} \\
            \vdots \\
            \digit{x}{\idxi -1} \\
            x
        \end{pmatrix} 
        &\mapsto 10^\idxi \cdot x - \sum_{\idxj = 1}^{\idxi - 1} 10^{\idxi - \idxj} \; \digit{x}{\idxj} - 1 + 10^{-\left(N + 1\right)} \\
        &=
        \innerProd{
        \underbrace{\begin{pmatrix}
            -10^{\idxi - 1} \\
            \vdots \\
            -10^{1} \\
            10^\idxi
        \end{pmatrix}}_{\mW^\top}
        }{
        \begin{pmatrix}
            \digit{x}{1} \\
            \vdots \\
            \digit{x}{\idxi -1} \\
            x
        \end{pmatrix}
        } \underbrace{- 1 + 10^{-\left(N + 1\right)}}_{\vb} \label{eq:digit-mlp-1}
    \end{align}
    \end{subequations}
    is an affine transformation.
    The indicator function in \cref{eq:r-eq-again}, however, is discontinuous and can thus not be implemented by a composition of continuous $\ReLU$ MLPs.
    Here, we take advantage of the fact that 
    \begin{equation}
        10^\idxi \cdot x - \sum_{\idxj = 1}^{\idxi - 1} 10^{\idxi - \idxj} \; \digit{x}{\idxj} - 1 + 10^{-\left(N + 1\right)} \in \underbrace{\left(-\infty, 0\right] \cup \left[10^{-\left(N + 1\right)}, \infty\right)}_{\sI}.
    \end{equation}
    The MLP 
    \begin{subequations}
    \begin{align}
        \mlp_\sI \left(z\right) &= 10^{N + 1} \left(\ReLUfunc{z} - \ReLUfunc{z - 10^{-\left(N + 1\right)}}\right) \\
        &= \begin{pmatrix}
            10^{N + 1} & -10^{N + 1}
        \end{pmatrix}
        \ReLUfunc{
        \begin{pmatrix}
            1 \\ 1
        \end{pmatrix} z +
        \begin{pmatrix}
            0 \\ -10^{-\left(N + 1\right)}
        \end{pmatrix}
        } \label{eq:digit-mlp-2}
    \end{align}
    \end{subequations}
    matches $\ind{\cdot > 0}$ on $\sI$, as we show next.
    See \cref{fig:ind-approx} for an illustration of the approximation.
    First, assume that $z \leq 0$.
    Then
    \begin{equation}
        \mlp_\sI \left(z\right) = 10^{N + 1} \left(\ReLUfunc{z} - \ReLUfunc{z - 10^{-\left(N + 1\right)}}\right) = 10^{N + 1} \left(0 - 0\right) = 0 = \ind{0 > 0}.
    \end{equation}
    On the contrary, assuming $z \geq 10^{-\left(N + 1\right)}$, we have that
    \begin{subequations}
    \begin{align}
        \mlp_\sI \left(z\right) 
        &= 10^{N + 1} \left(\ReLUfunc{z} - \ReLUfunc{z - 10^{-\left(N + 1\right)}}\right) \\
        &= 10^{N + 1} \left(z - \left(z - 10^{-\left(N + 1\right)}\right)\right) \\
        &= 10^{N + 1} \left(z - z + 10^{-\left(N + 1\right)}\right) \\
        &= 10^{N + 1} \left(10^{-\left(N + 1\right)}\right) \\
        &= 1 = \ind{z > 0}.
    \end{align}
    \end{subequations}
    We can therefore construct the MLP $\mlp$ computing \cref{eq:r-eq-again} on $\sI$ by a composition of \cref{eq:digit-mlp-1} (computing $z$ in \cref{eq:digit-mlp-2}) and the MLP $\mlp_\sI$ from \cref{eq:digit-mlp-2}.
\end{proof}

\begin{lemma} \label{lem:single-mlp}
    Let $N \in \N$.
    There exists an MLP $\mlp$ such that
    \begin{equation}
        \mlp\left(x\right) = \digitVec{x}
    \end{equation}
    for all $x \in \Q \cap \left(10^{-\left(N + 1\right)}, 10^{-1}\right]$ with $\digit{x}{\idxi} \in \set{0, 1}$ for $\idxi \in \NTo{N}$. 
\end{lemma}
\begin{proof}
    At a very high level, we will construct an $N$-layer MLP performing the transformations 
    \begin{equation}
        x \mapsto \begin{pmatrix}
            x \\
            x \\
            \vdots \\
            x
        \end{pmatrix} \mapsto \begin{pmatrix}
            \digit{x}{1} \\
            x \\
            \vdots \\
            x
        \end{pmatrix} \mapsto \begin{pmatrix} 
            \digit{x}{1} \\
            \digit{x}{2} \\
            \vdots \\
            x
        \end{pmatrix} \mapsto \cdots \mapsto \begin{pmatrix}
            \digit{x}{1} \\
            \digit{x}{2} \\
            \vdots \\
            \digit{x}{N}
        \end{pmatrix} = \digitVec{x}.
    \end{equation}
    By \cref{lemma:single-lemma-2}, all individual transformations can be performed exactly by a single-layer MLP.
    The composition of the $N$ layers results in the vector $\digitVec{x}$.

    The transformation $x \mapsto \begin{pmatrix}
            x \\
            x \\
            \vdots \\
            x
        \end{pmatrix}$ is a simple linear transformation.
    We now construct the $\ell\textsuperscript{th}$ layer $\vfunc_\ell$ of the MLP with parameters $\mW_\ell$ and $\vb_\ell$ (cf. \cref{def:mlp}), assuming that it has the input $\begin{pmatrix} 
            \digit{x}{1} \\
            \digit{x}{2} \\
            \vdots \\
            \digit{x}{\ell - 1} \\
            x \\
            \vdots \\
            x
        \end{pmatrix}$.
    The layer $\vfunc_\ell$ has to satisfy
    \begin{equation} \label{eq:single-mlp-layer-req}
        \vfunc\left(\begin{pmatrix} 
            \digit{x}{1} \\
            \digit{x}{2} \\
            \vdots \\
            \digit{x}{\ell - 1} \\
            \textcolor{ETHRed}{x} \\
            x \\
            \vdots \\
            x
        \end{pmatrix}\right) = 
        \begin{pmatrix} 
            \digit{x}{1} \\
            \digit{x}{2} \\
            \vdots \\
            \digit{x}{\ell - 1} \\
            \textcolor{ETHRed}{\digit{x}{\ell}} \\
            x \\
            \vdots \\
            x
        \end{pmatrix},
    \end{equation}
    i.e., it must copy all $\ngr - 1$ entries apart from the $\ell\textsuperscript{th}$ one.
    Thus, $\mW_{\idxk, \idxk'} = \ind{k = k'}$ for $\idxk \neq \ell$ and $\vb_{\idxk} = 0$ for all $\idxk \neq \ell$ (here, we write $\mW$ and $\vb$ for $\mW_\ell$ and $\vb_\ell$ to avoid clutter).
    To define the remaining $\ell\textsuperscript{th}$ row, we use \cref{lemma:single-lemma-2}.
    It tells us that defining
    \begin{subequations}
    \begin{align}
        \mW_{\ell, :} &\defeq \begin{pmatrix}
            10^{\ell - 1}& \cdots & 10^{1} & 10^\ell & 0 & \cdots & 0 \\
            10^{\ell - 1}& \cdots & 10^{1} & 10^\ell & 0 & \cdots & 0 
        \end{pmatrix} \\
        \evb_{\ell} &\defeq \begin{pmatrix}
            -1 + 10^{-\ngr} \\
            -1 + 10^{-\ngr} - 10^{-\ngr}
        \end{pmatrix} = \begin{pmatrix}
            -1 + 10^{-\ngr} \\
            -1
        \end{pmatrix}
    \end{align}
    \end{subequations}
    will result in the $\ell\textsuperscript{th}$ row of $\vfunc_\ell$ computing exactly $\digit{x}{\ell}$ after being multiplied by the matrix 
    \begin{equation}
        \mW' \defeq \begin{pmatrix}
            10^{\ngr} & 10^{\ngr}
        \end{pmatrix}.
    \end{equation}
    This is what \cref{eq:single-mlp-layer-req} requires.
    The parameters $\mW_{\ell, :}$ and $\evb_\ell$ represent the parameters of the affine transformation in \cref{lemma:single-lemma-2}.
    Note that the matrix $\mW'$ is not part of the original definition of the MLP.
    However, it can easily be absorbed into the matrix $\mW_{\ell + 1}$ in the actual implementation (at the cost of duplicating the size of the hidden state).
    We keep it here to make the presentation more intuitive. 
    Since any layer $\vfunc_\ell$ can be defined like this, and the final MLP is their composition, this finishes the proof.
\end{proof}

\begin{lemma} \label{lem:norm-a}
    Given $\va_{\tstep - 1}$ from \cref{lem:digit-0}, it holds that
    \begin{equation}
        \norm{\va_{\tstep - 1}}_1 = \frac{10^{\ngr - 1 - \tstep}}{\ngr - 1} \frac{1 - \frac{1}{10^{\ngr - 1}}}{1 - \frac{1}{10}}.
    \end{equation}
\end{lemma}
\begin{proof}
    By \cref{lem:digit-1}, we have that 
    \begin{equation}
        \va \defeq \va_{\tstep - 1} = \sum_{\idxj = \tstep - \ngr + 1}^{\tstep - 1} \frac{1}{\ngr - 1} 10^{-\idxj} \onehot{\sym_{\idxj}}
    \end{equation}
    We compute:
    \begin{subequations}
        \begin{align}
            \norm{\va}_1 
            &= \frac{1}{\ngr - 1} \sum_{\idxj = \tstep - \ngr + 1}^{\tstep - 1} 10^{-\idxj} \\
            &= \frac{1}{\ngr - 1} \sum_{\idxj' = 0}^{\tstep - 1 - \left(\tstep - \ngr + 1\right)} 10^{-\left(\idxj' + \tstep - \ngr + 1\right)} \\
            &= \frac{1}{\ngr - 1} 10^{-\left(\tstep - \ngr + 1\right)} \sum_{\idxj = 0}^{\ngr - 2} \frac{1}{10^{\idxj}} \\
            &= \frac{10^{-\left(\tstep - \ngr + 1\right)}}{\ngr - 1} \frac{1 - \frac{1}{10^{\ngr - 1}}}{1 - \frac{1}{10}} \\
            &= \frac{10^{\ngr - 1 - \tstep}}{\ngr - 1} \frac{1 - \frac{1}{10^{\ngr - 1}}}{1 - \frac{1}{10}}.
        \end{align}
    \end{subequations}
\end{proof}

\begin{corollary} \label{cor:norm-a}
    Given $\va \defeq \va_{\tstep - 1}$ from \cref{lem:digit-0} and $\va' \defeq \va'_{\tstep - 1}$ from \cref{lem:digit-1}, it holds that
    \begin{equation}
        \frac{1}{\norm{\va}_1} \cdot 10 \cdot \frac{1 - \frac{1}{10^{\ngr - 1}}}{1 - \frac{1}{10}} \cdot \va = \va'.
    \end{equation}
\end{corollary}
\begin{proof}
    By \cref{lem:norm-a}, we can write 
    \begin{equation}
        \norm{\va}_1 = \frac{10^{\ngr - 1 - \tstep}}{\ngr - 1} \frac{1 - \frac{1}{10^{\ngr - 1}}}{1 - \frac{1}{10}} = \frac{10^{\ngr - \tstep}}{\ngr - 1} \cdot Z
    \end{equation}
    and 
    \begin{equation}
        \frac{1}{\norm{\va}_1} = \frac{\ngr - 1}{10^{\ngr - \tstep}} \cdot \frac{1}{Z}
    \end{equation}
    where $Z \defeq \frac{1}{10} \cdot \frac{1 - \frac{1}{10^{\ngr - 1}}}{1 - \frac{1}{10}}$ is a constant independent of $\tstep$.
    Then
    \begin{subequations}
        \begin{align}
            \frac{1}{\norm{\va}_1} \cdot \frac{1}{10} \cdot \frac{1 - \frac{1}{10^{\ngr - 1}}}{1 - \frac{1}{10}} \cdot \va
            &= \frac{\ngr - 1}{10^{\ngr - \tstep}} \cdot \frac{1}{Z} \cdot \underbrace{\frac{1}{10} \cdot \frac{1 - \frac{1}{10^{\ngr - 1}}}{1 - \frac{1}{10}}}_{Z} \cdot \va \\
            &= \frac{\ngr - 1}{10^{\ngr - \tstep}} \cdot \frac{1}{Z} \cdot Z \cdot \va \\
            &= \frac{\ngr - 1}{10^{\ngr - \tstep}} \cdot \va \\
            &= \va'
        \end{align}
    \end{subequations}
\end{proof}

\begin{lemma} \label{lem:single-final}
    Let $\alphabet$ be an alphabet.
    Given the transformer parameters and functions defined in \cref{eq:single-static}--\cref{eq:single-value}, there exists an MLP $\fTransf$ (whose inputs are $\norm{\cdot}_1$-normalized) that maps the contextual representations $\va$ of $\strlt$ into a one-hot encoding of $\str_{\tstep - \ngr + 1}^{\tstep - 1}$ for any string $\str \in \kleene{\alphabet}$ and any $\tstep \in \NTo{|\str|}$.
\end{lemma}
\begin{proof}
    By \cref{lem:digit-1}, we have that 
    \begin{equation}
        \va = \sum_{\idxj = \tstep - \ngr + 1}^{\tstep - 1} \frac{1}{\ngr - 1} 10^{-\idxj} \onehot{\sym_{\idxj}}
    \end{equation}
    and that $\digit{\eva'_\sym}{\idxi} = 1 \iff \sym_{\tstep - \ngr + \idxi} = \sym$ for $\va' \defeq \frac{\ngr - 1}{10^{\ngr - \tstep}} \; \va$ (where $\eva'_\idxi = 0$ for $\idxi > \ngr - 1$, from the same lemma).
    We can express the entries of the one-hot encoding of the history $\str_{\tstep - \ngr + 1}^{\tstep - 1}$, $\onehot{\str_{\tstep - \ngr + 1}^{\tstep - 1}}$, as 
    \begin{equation}
        \onehot{\str_{\tstep - \ngr + 1}^{\tstep - 1}}_{\sym_{\tstep - \ngr + 1} \ldots \sym_{\tstep - 1}} = 1 \iff \digit{\eva'_{\sym_{\tstep - \ngr + 1}}}{1} \wedge \cdots \wedge \digit{\eva'_{\sym_{\tstep - 1}}}{\ngr - 1}.
    \end{equation}
    By \cref{lem:single-mlp}, the vectors $\digitVec{\eva'_{\sym_{\tstep - \ngr + 1}}}, \ldots, \digitVec{\eva'_{\sym_{\tstep - 1}}}$ can be computed by an $\ngr - 1$-layer MLP. 
    Each of these vectors is of size $\ngr - 1$ and, among others, contains the values $\digit{\eva'_{\sym_{\tstep - \ngr + 1}}}{1}, \ldots, \digit{\eva'_{\sym_{\tstep - 1}}}{\ngr - 1}$.
    Since the entries $\onehot{\str_{\tstep - \ngr + 1}^{\tstep - 1}}_{\sym_{\tstep - \ngr + 1} \ldots \sym_{\tstep - 1}}$ can be expressed as the results of the logical \texttt{AND} operation, their computation can be performed by a single-layer $\ReLU$ MLP as per \cref{fact:and}.
    The MLP $\fTransf$ can therefore be constructed as a composition of three functions:
    \begin{enumerate}
        \item The scaling $\va \mapsto \frac{1}{\norm{\va}_1} \cdot \frac{1}{10} \cdot \frac{1 - \frac{1}{10^{\ngr - 1}}}{1 - \frac{1}{10}} \cdot \va$, which results in $\va'$ by \cref{cor:norm-a}.
        \item The concatenation of the $\bosnsymbols$ $\ngr - 1$-layer MLPs computing $\digitVec{\eva'_{\sym}}$ for all $\sym \in \bosalphabet$. This results in $\left(\ngr - 1\right) \bosnsymbols$ binary values altogether.
        \item The MLP performing the \texttt{AND} operation between the entries of $\digitVec{\eva'_{\sym}}$.
    \end{enumerate}
    This finishes the proof.
\end{proof}

\singleLayerSingleHeadTheorem*
\begin{proof}
    To show that there exists a weakly equivalent single-layer-single-head transformer LM to any \ngram LM, we combine the lemmata in this section.
    Let $\tf$ be a transformer LM defined with the parameters and functions defined in \cref{eq:single-static}--\cref{eq:single-value}.
    By \cref{lem:digit-0}, the representations $\va_{\tstep - 1} = \sum_{\idxj = \tstep - \ngr + 1}^{\tstep - 1} \frac{1}{\ngr - 1} 10^{-\idxj} \onehot{\sym_{\idxj}}$ computed by $\tf$ contain information about the symbols and their positions in the history $\str_{\tstep - \ngr + 1}^{\tstep -1}$.
    By \cref{lem:single-final} then, $\va$ can be mapped to the one-hot encoding of the history with a $\ngr - 1$-layer MLP $\fTransf$.
    This one-hot encoding can then be used to index (the logits of) the probabilities stored in the output matrix $\outMtx$ defining a weakly equivalent transformer.
\end{proof}

\section{Sparse Attention} \label{sec:proofs-sparse-attention}

We now prove a lemma analogous to \cref{lemma:hard-attention-lemma-3}.
It shows that a sparse attention transformer head can isolate a particular symbol in the string.
First, define the following position-augmented symbol representation function of the transformer head $h$:
\begin{equation}
    \posInEmbeddingFun{\sym, \tstep} \defeq \begin{pmatrix}
        \onehot{\sym} \\ 
        \zero_{\bosnsymbols} \\
        1 \\
        \tstep
      \end{pmatrix}  \in \set{0, 1}^{2 \bosnsymbols + 2}
  \end{equation}
  and the scoring function
\begin{equation}
    \tfscorefun\left(\vq, \vk\right) \defeq -\abs{\inner{\vq}{\vk}}.
\end{equation}
Here, the position-augmented symbol representation function $\posInEmbedding$ can again be implemented by concatenating or summing a symbol- and a position-specific component.
Lastly, we define the transformation matrices
\begin{subequations}
\begin{alignat}{2}
    \qTransf\left(\vx\right) & \defeq \mQ \vx + \vb_\qTransf, \qquad &&\mQ \in \R^{2 \times \left(2 \bosnsymbols + 2\right)}, \vb_\qTransf \in \R^2,            \\
    \kTransf\left(\vx\right) & \defeq \mK \vx, \qquad &&\mK \in \R^{2 \times \left(2 \bosnsymbols + 2\right)},                                \\
    \vTransf\left(\vx\right) & \defeq \mV \vx, \qquad &&\mV \in \R^{\left(2 \bosnsymbols + 2\right) \times \left(2 \bosnsymbols + 2\right)},                              \\
    \oTransf\left(\vx\right) & \defeq \mO \vx, \qquad &&\mO \in \R^{\left(2 \bosnsymbols + 2\right) \times \left(2 \bosnsymbols + 2\right)},
\end{alignat}
\end{subequations}
\begin{subequations}
\begin{align}
    \mQ_{:, 2 \bosnsymbols + 1: 2 \bosnsymbols + 2}             & \defeq \mI_2 \\
    \vb_{\qTransf}             & \defeq \begin{pmatrix}
        0 \\ -h 
    \end{pmatrix} \\
    \mK_{:, 2 \bosnsymbols + 1: 2 \bosnsymbols + 2}             & \defeq \begin{pmatrix}
        0 & 1 \\ -1 & 0
    \end{pmatrix},                   \\
    \mV_{\bosnsymbols + 1: 2 \bosnsymbols, 1: \bosnsymbols} & \defeq \mI_\bosnsymbols \\
    \mO_{1: \bosnsymbols, 1:\bosnsymbols} &\defeq -\mI_{|\bosalphabet|}
\end{align}
\end{subequations}
where the unspecified elements of $\mQ, \mK$, $\mV$, $\mO$ are $0$.

\begin{lemma} \label{lem:sparse-attention-lemma-1}
    Let $\alphabet$ be an alphabet and $\str \in \kleene{\alphabet}$.
    For any $\tstep \in \NTo{|\str|}$, a transformer head defined with the parameters above outputs
    \begin{equation} 
        \vz_{\tstep - 1} = \begin{pmatrix}
            \zero_{\bosnsymbols}                                                     \\
            \onehot{\sym_{\tstep - 1 - h}} \\
            \zero_{2}
        \end{pmatrix}.
    \end{equation}
    In particular, this means that the output $\vz_{\tstep - 1}$ at time step $\tstep - 1$ contains the one-hot encoding of the symbol at position $\tstep - h - 1$.
\end{lemma}
\begin{proof}
    The construction is largely identical to the one in \cref{thm:transformers-n-gram-label}, with one crucial difference: It relies on simpler, but \emph{unbounded} positional encodings and a less-standard, but still easily implementable attention scoring function in the form of an MLP.

    It is easy to see that the query and value transformations result in: 
    \begin{subequations}
    \begin{align}
        \vq_{\tstep - 1} &= \qTransf\left(\posInEmbeddingFun{\symt, \tstep}\right) = \begin{pmatrix}
            1 \\ \tstep - 1 - h
        \end{pmatrix} \\
        \vk_\idxj &= \kTransf\left(\posInEmbeddingFun{\sym_\idxj, \idxj}\right) = \begin{pmatrix}
            -\idxj \\ 1
        \end{pmatrix}.
    \end{align}
    \end{subequations}
    Thus, we get that 
    \begin{subequations}
    \begin{align}
        \tfscorefun\left(\vq_{\tstep - 1}, \vk_\idxj\right) 
        &= -\abs{\inner{\vq_{\tstep - 1}}{\vk_\idxj}} \\
        &= -\abs{\inner{\begin{pmatrix}
            1 \\ \tstep - 1 - h
        \end{pmatrix}}{\begin{pmatrix}
            -\idxj \\ 1
        \end{pmatrix}}} \\
        &= -\abs{\idxj - \left(\tstep - h - 1\right)}.
    \end{align}
    \end{subequations}
    $\tfscorefun$ clearly has a unique maximum at $\idxj^\ast = \tstep - 1 - h$.
    Moreover, by construction, $\tfscorefun\left(\vq_\tstep, \vk_{\idxj^\ast}\right) \geq \tfscorefun\left(\vq_\tstep, \vk_{\idxj}\right) + 1$ for any $\idxj \neq \idxj^\ast$.
    This is a crucial property of the scoring function and one that allows sparsemax to \emph{uniquely} attend to $\idxj^\ast$; by \cref{lem:sparsemax-unique}, it holds that
    \begin{equation}
        \sparsemax\left(\tfscorefun\left(\vq_{\tstep - 1}, \vk_1\right), \ldots, \tfscorefun\left(\vq_{\tstep - 1}, \vk_{\tstep - 1}\right)\right) = \hardmax\left(\tfscorefun\left(\vq_{\tstep - 1}, \vk_1\right), \ldots, \tfscorefun\left(\vq_{\tstep - 1}, \vk_{\tstep - 1}\right)\right),
    \end{equation}
    meaning that 
    \begin{equation}
        \sparsemax\left(\tfscorefun\left(\vq_{\tstep - 1}, \vk_1\right), \ldots, \tfscorefun\left(\vq_{\tstep - 1}, \vk_{\tstep - 1}\right)\right)_\idxj = \ind{\idxj = \tstep - 1 - h}
    \end{equation}
    as in the proof of \cref{lemma:hard-attention-lemma-3}.
    This is exactly the same as \cref{eq:hardmax-max}.
    Since $\oTransf$ and $\vTransf$ result in the same vectors as in \cref{lemma:hard-attention-lemma-3}, the remainder of the proof is the same as in \cref{lemma:hard-attention-lemma-3}.
\end{proof}

\begin{lemma} \label{lem:sparsemax-unique}
    Let $\vx \in \R^\hiddDim$.
    If $\displaystyle \max_{\idxd = 1}^\hiddDim{\evx_\idxd} \geq \max_{\substack{\idxd = 1 \\ \idxd \notin \argmax{\left(\vx\right)}}}^\hiddDim{\evx_\idxd} + 1$, then $\sparsemax\left(\vx\right) = \hardmax\left(\vx\right)$.
\end{lemma}
\begin{proof}
    Let $\vx \in \R^\hiddDim$ and let $\evx_{\left(1\right)} \geq \evx_{\left(2\right)} \geq \ldots \geq \evx_{\left(\hiddDim\right)}$ be the non-increasing entries of $\vx$.
    Due to the additive invariance of the softmax (\citet[][Proposition 2]{sparsemax}), we can assume that $\evx_{\left(1\right)} = 0$ and $\evx_{\left(2\right)} \leq -1$.
    By \citet[][Proposition 1]{sparsemax}, 
    \begin{equation}
        \sparsemax\left(\vx\right)_\idxd = \max\left(0, \evx_\idxd - \tau\left(\vx\right)\right),
    \end{equation}
    where 
    \begin{equation}
        \tau\left(\vx\right) \defeq \frac{\sum_{\idxj = 1}^{k\left(\vx\right)} \evx_{\left(\idxj\right)} - 1}{k\left(\evx\right)}
    \end{equation}
    and
    \begin{equation} \label{eq:k-cond}
        k\left(\vx\right) \defeq \max\left(k \in \NTo{\hiddDim} \mid 1 + k \evx_{\left(k\right)} > \sum_{\idxj = 1}^{k} \evx_{\left(\idxj\right)} \right).
    \end{equation}

    It suffices to show that $k\left(\vx\right) = 1$.
    For $k = 1$, we get
    \begin{equation}
        1 + 1 \cdot \evx_{\left(1\right)} = 1 + 1 \cdot 0 = 1 = 1 \geq 0 = \evx_{\left(1\right)}.
    \end{equation}
    For $k = 2$, we get
    \begin{subequations}
    \begin{align}
        1 + 2 \cdot \evx_{\left(2\right)} 
        &= 1 + \evx_{\left(2\right)} + \evx_{\left(2\right)} \\
        &\leq 1 - 1 + \evx_{\left(2\right)} \\
        &= \evx_{\left(2\right)} \\
        &= \evx_{\left(1\right)} + \evx_{\left(2\right)} \\
        &= \sum_{\idxj = 1}^{2} \evx_{\left(\idxj\right)},
    \end{align}
    \end{subequations}
    meaning that the condition from \cref{eq:k-cond} is not fulfilled for $k = 2$. This implies that $k\left(\vx\right) = 1$ and $\tau\left(\vx\right) = \evx_{\left(1\right)} - 1$.
    Thus, we get that 
    \begin{equation}
        \sparsemax\left(\vx\right)_{\left(1\right)} = \max\left(0, \evx_{\left(1\right)} - \evx_{\left(1\right)} + 1\right) = 1
    \end{equation}
    and  
    \begin{equation}
        \sparsemax\left(\vx\right)_{\left(\idxd\right)} = \max\left(0, \evx_{\left(\idxd\right)} - \evx_{\left(1\right)} + 1\right) = \max\left(0, \underbrace{\evx_{\left(1\right)}}_{\leq -1} - 0 + 1\right) = 0
    \end{equation}
    for $\idxd > 1$.
\end{proof}

This allows us to prove the main theorem for sparse-attention transformer LMs.
\singleLayerSparseTheorem*
\begin{proof}
    \cref{lem:sparse-attention-lemma-1} shows how individual heads of the transformer can identify the symbols in the position of interest.
    $\ngr - 1$ of them can identify the entire history.
    The proof then follows the same reasoning as that of \cref{thm:transformers-n-gram-label}.
\end{proof}
Adapting the same proof strategy to \cref{thm:transformers-n-gram-label-multi-layer} would naturally result in an analogous result for $\ngr - 1$ layers and a single head.

Notice that \cref{lem:sparse-attention-lemma-1} requires different and less standard positional encodings, which are, crucially, unbounded.
Constructing a sparse attention transformer with bounded positional encodings seems more difficult; the contextual representations would in that case either converge or be non-unique with $\tstep \to \infty$ and since the sparsemax always contracts \citep[][Proposition 2]{sparsemax}, attending to individual positions would be difficult.
While the positional encodings and the scoring function used in the proof of \cref{lem:sparse-attention-lemma-1} are somewhat less standard than those used in \cref{lemma:hard-attention-lemma-3}, similar positional encodings and the same scoring function have been used in theoretical analyses before and even in practical implementations \citep{perez-turing}.

\end{document}